\documentclass{applemlr}

\usepackage{amsmath}
\usepackage{amssymb}
\usepackage{mathtools}
\usepackage{amsthm}
\usepackage{pifont} %

\newtheorem{proposition}{Proposition}[section]

\newtheorem{corollary}[proposition]{Corollary}
\newtheorem{lemma}[proposition]{Lemma}

\theoremstyle{remark}
\newtheorem{remark}[proposition]{Remark}

\tcolorboxenvironment{proposition}{
  enhanced, breakable, frame hidden,
  colback=titlebg, arc=4pt,
  left=3mm, right=3mm, top=2mm, bottom=2mm,
  before skip=8pt plus 2pt, after skip=8pt plus 2pt
}
\tcolorboxenvironment{theorem}{
  enhanced, breakable, frame hidden,
  colback=titlebg, arc=4pt,
  left=3mm, right=3mm, top=2mm, bottom=2mm,
  before skip=8pt plus 2pt, after skip=8pt plus 2pt
}
\tcolorboxenvironment{corollary}{
  enhanced, breakable, frame hidden,
  colback=titlebg, arc=4pt,
  left=3mm, right=3mm, top=2mm, bottom=2mm,
  before skip=8pt plus 2pt, after skip=8pt plus 2pt
}
\tcolorboxenvironment{lemma}{
  enhanced, breakable, frame hidden,
  colback=titlebg, arc=4pt,
  left=3mm, right=3mm, top=2mm, bottom=2mm,
  before skip=8pt plus 2pt, after skip=8pt plus 2pt
}
\tcolorboxenvironment{definition}{
  enhanced, breakable, frame hidden,
  colback=titlebg, arc=4pt,
  left=3mm, right=3mm, top=2mm, bottom=2mm,
  before skip=8pt plus 2pt, after skip=8pt plus 2pt
}
\tcolorboxenvironment{remark}{
  enhanced, breakable, frame hidden,
  colback=titlebg, arc=4pt,
  left=3mm, right=3mm, top=2mm, bottom=2mm,
  before skip=8pt plus 2pt, after skip=8pt plus 2pt
}

\newtcolorbox{keyinsight}{
  enhanced, breakable, frame hidden,
  colback=titlebg, arc=4pt,
  borderline west={2.5pt}{0pt}{bscolor},
  left=5mm, right=3mm, top=2mm, bottom=2mm,
  before skip=10pt plus 2pt, after skip=10pt plus 2pt
}

\usepackage{algorithm}
\usepackage{algpseudocode}
\usepackage{enumitem}

\usepackage{wrapfig}
\usepackage{booktabs}
\usepackage{colortbl}  %
\usepackage{etoc}

\usepackage[textsize=tiny]{todonotes}

\title{Embarrassingly Simple Self-Distillation Improves Code Generation}

\author{
  Ruixiang Zhang$^{*}$
  \quad Richard He Bai$^{*}$
  \quad Huangjie Zheng$^{*}$
  \quad Navdeep Jaitly
  \quad Ronan Collobert \\[0.5em]
  Yizhe Zhang$^{*}$
}

\newcommand{\sd}{simple self-distillation}
\newcommand{\Sd}{Simple self-distillation}
\newcommand{\SD}{SSD}
\usepackage{soul}
\newcommand{\cpanel}[2]{\hyperref[#1]{\Cref*{#1}#2}}
\newcommand{\Cpanel}[2]{\hyperref[#1]{\Cref*{#1}#2}}

\newcommand{\github}{\raisebox{-1.5pt}{\includegraphics[height=1.05em]{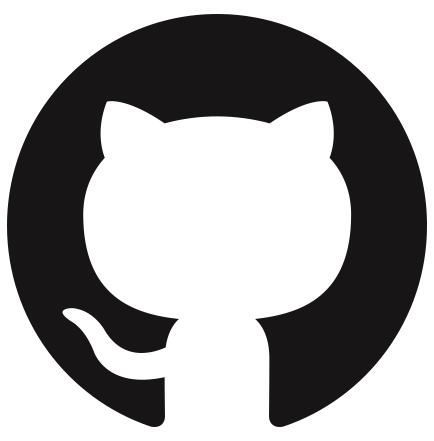}}}

\affiliation{
  \hspace{-2.5pt}Apple\\[0.5em]
  \github\quad \url{https://github.com/apple/ml-ssd}
}

\abstract{
Can a large language model~(LLM) improve at code generation using only its own raw outputs, without a verifier, a teacher model, or reinforcement learning? We answer in the affirmative with simple self-distillation (\SD): sample solutions from the model with certain temperature and truncation configurations, then fine-tune on those samples with standard supervised fine-tuning. \SD{} improves Qwen3-30B-Instruct from 42.4\% to 55.3\% pass@1 on LiveCodeBench\,v6, with gains concentrating on harder problems, and it generalizes across Qwen, Llama, and GPT-OSS models at 4B--30B scale, including both instruct and thinking variants.
To understand why such a simple method can work, we trace these gains to a \emph{precision-exploration conflict} in LLM decoding and show that \SD{} reshapes token distributions in a context-dependent way, suppressing distractor tails where precision matters while preserving useful diversity where exploration matters. Taken together, \SD{} offers a complementary post-training direction for improving LLM code generation.
}

\metadata[Correspondence]{Ruixiang Zhang (\texttt{ruixiangz@apple.com}), Richard He Bai (\texttt{richardbai@apple.com}), Huangjie Zheng (\texttt{huangjie.zheng@apple.com}), Yizhe Zhang (\texttt{yizzhang@apple.com})}
\date{\sffamily April 1, 2026\quad{\small $^{*}$ Equal contribution}}

\renewcommand{\mymaketitle}{%
  \tcbset{enhanced,frame hidden}
  \tcbset{colback=titlebg}
  \tcbset{left=0.5cm}
  \tcbset{right=0.5cm}
  \tcbset{top=0.2cm}
  \tcbset{bottom=0.3cm}
  \tcbset{arc=16pt}
  \tcbset{before skip=0pt}
  \tcbset{grow to left by=1.5pt}
  \tcbset{grow to right by=1.5pt}
  \tcbset{overlay={\node[
    anchor=north west,
    at= (frame.north west),
    xshift= 0.5cm,
    yshift= -0.5cm] {\includegraphics[width=0.5cm]{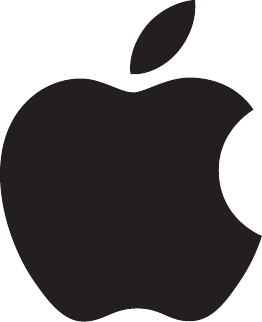}};}}
  \begin{tcolorbox}
    \setlength{\parindent}{0cm}
    \setlength{\parskip}{0.5cm}
    {
      \setlength{\parskip}{0cm}
      \raggedright
      \nohyphens
      {
        \vskip 1.2cm
        \titlelist\par
      }
      \vskip 0.35cm
      \authorlist\par
      \vskip 0.5em
      \affiliationlist\par
      \contributionlist\par
      \vskip 1.0em
      \abstractlist\par
    }
    \vskip 0.3cm
    {
      \setlength{\parskip}{4pt}
      \ifdefempty{\metadatalist}{\vspace*{0.65cm}}{\metadatalist\par}
    }
  \end{tcolorbox}
  \tcbset{reset}
  \FloatBarrier
}

\begin{document}

\maketitle

\vspace{0.1cm}
{%
\noindent
\centering
\includegraphics[width=\textwidth]{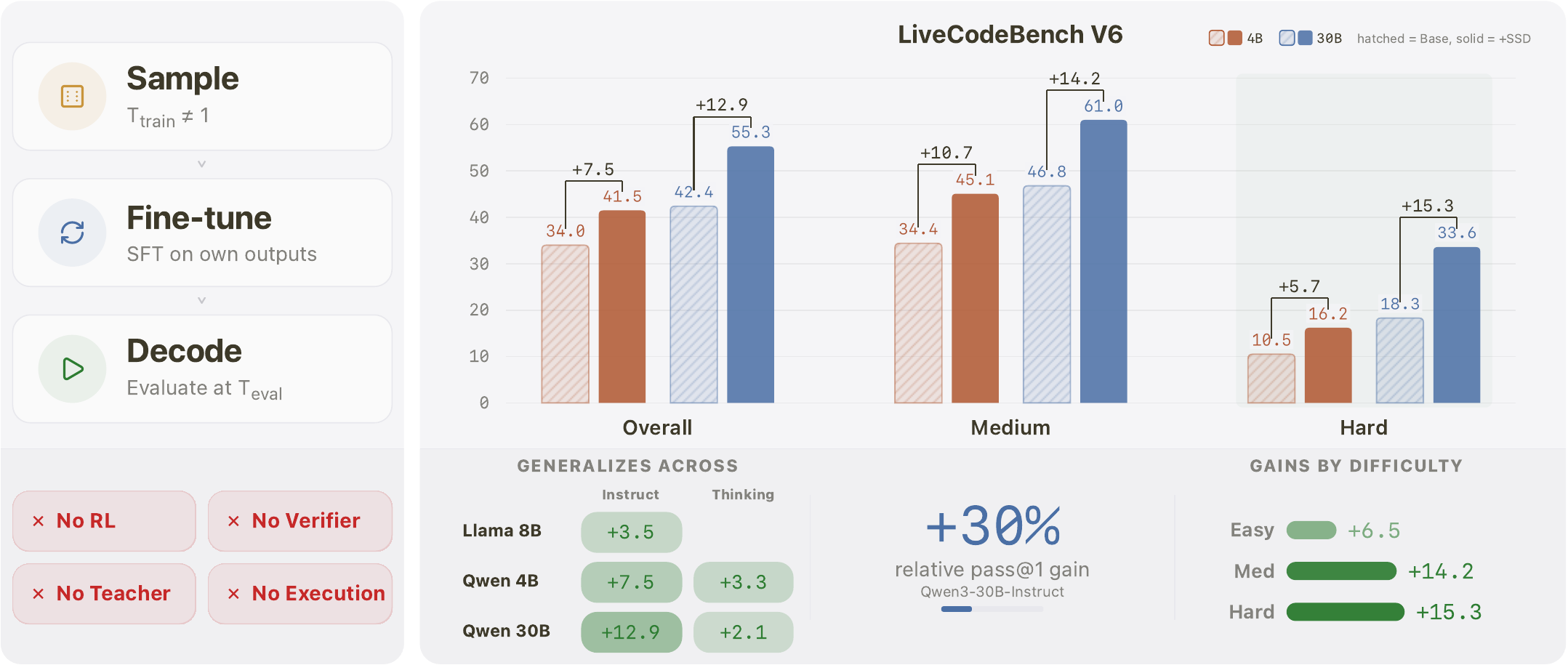}\\[0.1cm]
{\small
\captionof{figure}{\textbf{\Sd{}~(\SD) is embarrassingly simple, yet yields substantial LiveCodeBench\,v6 gains across six models spanning three families and multiple scales, with both instruct and thinking variants.}
Left: \SD{} samples from the base model with training-time decoding temperature $T_{\textsf{train}}$, fine-tunes on its own raw outputs, and then decodes at evaluation time with $T_{\textsf{eval}}$; it uses no RL, verifier, teacher, or code execution environment.
Right: LiveCodeBench\,v6 pass@1 for Qwen3-4B-Instruct and Qwen3-30B-Instruct on the Overall, Medium, and Hard splits (orange = 4B, blue = 30B; hatched = base, solid = +\SD{}).
The footer highlights the broader pattern: all six evaluated models improve, Qwen3-30B-Instruct gains +30\% relative pass@1, and the largest gains occur on harder problems.
}
\label{fig:teaser}
}%
}

\clearpage

\addtocontents{toc}{\protect\setcounter{tocdepth}{-1}}

\section{Introduction}
\label{sec:introduction}
As LLMs are deployed to increasingly difficult coding tasks, the supply of high-quality supervised signal has become a binding constraint: human-written solutions \citep{chen2021evaluating, austin2021program, hendrycks2021apps} are expensive to produce, and synthetic data pipelines require either a stronger teacher model \citep{hinton2015distilling, kim2016sequence, hsieh2023distilling, agarwal2024onpolicy} or execution-based verification for every training problem \citep{li2022alphacode, le2022coderl, singh2024beyond, liu2025rstarcoder}. Teacher-based distillation also inherits the ceiling of the teacher, while reinforcement learning with verifiable reward remains operationally complex and can be unstable, even in recent RL-based reasoning and coding pipelines \citep{he2026far, shao2024deepseekmath, guo2025deepseekr1, openai2025competitive}. Unsupervised alternatives that use intrinsic rewards such as majority voting or entropy minimization \citep{zuo2025ttrl, agarwal2025entropy} have shown early promise but face reward hacking and collapse under extended training \citep{zhang2025nofreelunch}. This raises a natural question: can a model improve itself without leveraging any external labeled data or verification at all?

We show the answer is yes. Our method, \emph{\sd} (\SD), is embarrassingly simple: sample solutions from the base model with specified temperature and truncation, then fine-tune on those raw, unverified samples via standard cross-entropy loss. This method requires only a set of problem prompts and the model itself: no human-labeled solutions, no reference answers, no teacher model, no reward model, no verifier, no execution environment, and no reinforcement learning of any kind.

\begin{table}[t]
\caption{Comparison of training paradigms.}
\label{tab:method_comparison}
\centering
\includegraphics[width=0.8\textwidth]{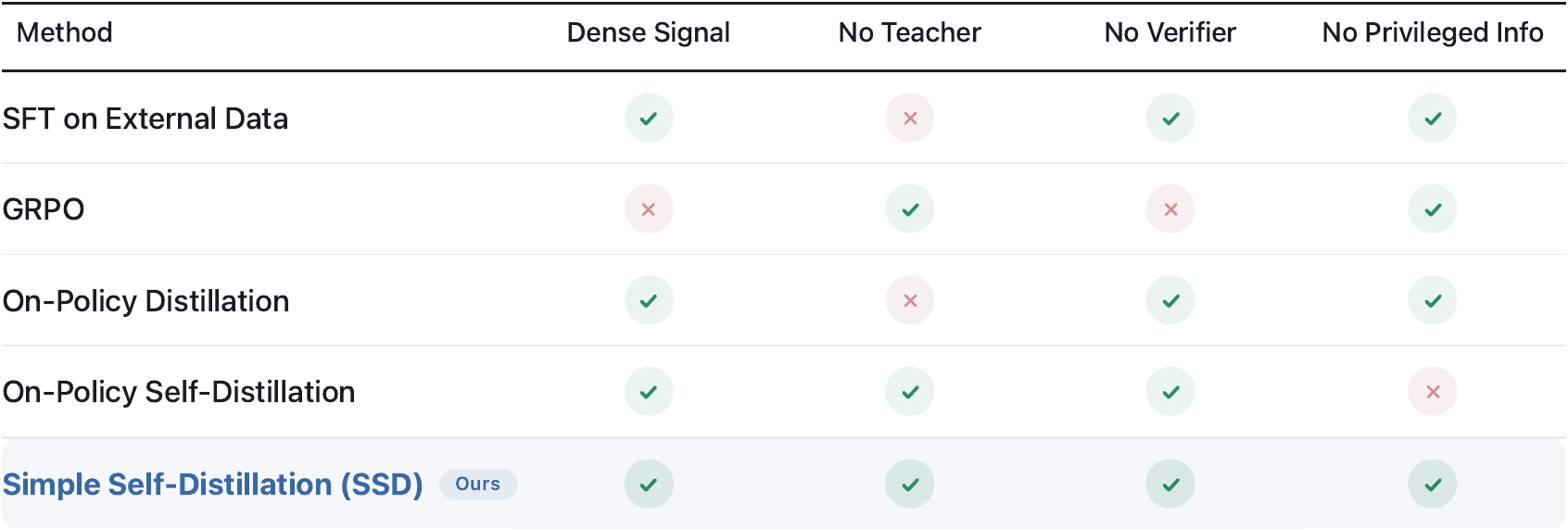}
\end{table}

Surprisingly, it works. \SD{} improves Qwen3-30B-Instruct from 42.4\% to 55.3\% pass@1 on LiveCodeBench\,v6~\citep{jain2024livecodebench}, with especially large gains on hard problems. Coverage improvements are larger still: hard-problem pass@5 rises from 31.1\% to 54.1\%, suggesting that \SD{} preserves useful exploration across solution branches instead of only sharpening a single dominant mode. These gains are not model-specific: \SD{} generalizes across six models spanning three families, multiple scales, and both instruct and thinking models.

We study why such a simple method can work in the domain of code generation, which serves as a particularly useful testbed because the task structure makes the underlying mechanism especially visible. Code interleaves \emph{fork} positions, where several continuations are genuinely plausible and may correspond to different solution approaches, with \emph{lock} positions, where syntax and semantics leave little ambiguity but a low-probability distractor tail still remains. These two context types make contradictory demands on decoding temperature~$T_{\textsf{eval}}$. Lowering~$T_{\textsf{eval}}$ secures locks but starves forks of diversity, while raising it enables exploration at forks but lets distractors flood back in at locks. The best global decoding setting is therefore necessarily a compromise; we call this tension the \emph{precision-exploration conflict}.

\SD{} can be understood through this lens. Training on temperature-shifted, truncated samples implicitly reshapes the model's distributions in a context dependent way (we formalize this later as \textit{support compression} and \textit{within support reshaping} in Equation \eqref{eq:three_term_alt}): it suppresses distractors most aggressively at locks while preserving useful diversity at forks. Evidence from controlled simulation and real-model analysis (\Cref{sec:mechanism:toy}), together with theoretical analysis (\Cref{sec:mechanism:objective}), supports this account and explains why changing only the decoding settings cannot recover the gains. More broadly, this suggests that existing code models contain capability not realized under fixed decoding alone.

Our contributions are threefold. First, we show that \SD{} can substantially improve code generation models' performance using only their own unverified outputs, without any external teacher, verifier, reward model, or labeled solutions. Second, we identify the \emph{precision-exploration conflict} and argue that it is the key mechanism behind \SD{}. Third, we support this mechanism with aligned evidence from controlled simulation, real-model analysis, and theory.

\section{Embarrassingly Simple Self-Distillation (\SD)}
\label{sec:method}

\paragraph{Data synthesis.}
We write $T$ for the temperature and $\rho$ for the truncation configuration, namely the top-$k$ and top-$p$ used in decoding (see \Cref{app:decoding_pipeline} for the precise decoding procedure).
Given a \textit{frozen} pre-trained LLM $p_{\theta}$ and a set of prompts $\mathcal{X}$, we sample $N$ candidate solutions per problem:
\begin{equation}
y \sim \textsf{Decode}_{T_{\textsf{train}},\, \rho_{\textsf{train}}}\!\big[p_{\theta}(\cdot \mid x)\big]
\end{equation}
The solutions are not verified in any way: no execution, no test cases, no filtering by correctness.
The raw outputs form the \sd{} dataset $\mathcal{D}_{\textsf{SSD}}$.
In practice, $N{=}1$ (a single sample per prompt) already suffices (\Cref{sec:experiments}).

\paragraph{Training.}
We fine-tune the model on $\mathcal{D}_{\textsf{SSD}}$ with standard supervised fine-tuning~(SFT):
\begin{equation}
\mathcal{L}(\theta) = - \mathbb{E}_{(x,y) \sim \mathcal{D}_{\textsf{SSD}}} \sum_{t=1}^{|y|} \log p_\theta(y_t \mid x, y_{<t})
\end{equation}

\paragraph{Inference.}
The fine-tuned model $p_{\theta^*}$ is deployed with evaluation-time decoding configuration $(T_{\textsf{eval}}, \rho_{\textsf{eval}})$:
\begin{equation}
\hat{y} \sim \textsf{Decode}_{T_{\textsf{eval}},\, \rho_{\textsf{eval}}}\!\big[p_{\theta^*}(\cdot \mid x)\big]
\end{equation}

\section{Experiments}
\label{sec:experiments}

\subsection{Experimental Setup}

\textbf{Models.}
We evaluate \SD{} on six models spanning three families (Llama, Qwen, GPT-OSS), scales from 4B to 30B, and two reasoning styles (instruct, thinking):
Llama-3.1-8B-Instruct (dense, 8B),
Qwen3-4B-Instruct-2507 (dense, 4B; hereafter \emph{Qwen3-4B-Instruct}),
Qwen3-4B-Thinking-2507 (dense, 4B; hereafter \emph{Qwen3-4B-Thinking}),
Qwen3-30B-A3B-Instruct-2507 (MoE, 30B total / 3B active; hereafter \emph{Qwen3-30B-Instruct}),
Qwen3-30B-A3B-Thinking-2507 (MoE, 30B / 3B active; hereafter \emph{Qwen3-30B-Thinking}), and
GPT-OSS-20B (MoE, 21B total / 3.6B active; hereafter \emph{GPT-OSS-20B}).
We apply \SD{} to each of these \emph{base} models.

\textbf{Data synthesis.}
We use the seed subset of the rSTARcoder dataset~\citep{liu2025rstarcoder}, de-duplicated to yield ${\sim}$10K unique competitive programming problems. For each prompt we sample a single solution from the base model using the per-model generation decoding configuration in \Cref{tab:sd_settings}.
We apply only minimal syntactic filtering to remove empty responses and single line stubs, meaning there is absolutely \textbf{no correctness signal} used.
Generation uses vLLM~\citep{kwon2023vllm} with 128K maximum sequence length limit.

\textbf{Training.}
We fine-tune with Megatron-LM\footnote{\url{https://github.com/NVIDIA/Megatron-LM}} on 8$\times$B200 GPUs (EP${=}$8 for MoE models), using AdamW with cosine decay (peak LR $5 \times 10^{-6}$), global batch size 32, sequence length 65{,}536, and 2{,}500 iterations for instruct models and 300 iterations for thinking models.
The learning rate is warmed-up with 250 and 50 iterations, respectively.

\textbf{Evaluation.}
Our primary benchmark is LiveCodeBench\,v6 (LCB\,v6; Feb--May 2025, following the version split adopted by recent model
  releases~\footnote{\url{https://huggingface.co/Qwen/Qwen3-4B-Instruct-2507}}; stratified by easy/medium/hard).
We report LCB\,v5 (Aug 2024--Feb 2025, following the version split adopted by rSTARcoder~\citep{liu2025rstarcoder}) as a secondary confirmation. The primary metric is pass@1; we also report pass@5 and per-difficulty breakdowns.
Base-model results use each model's officially recommended sampling parameters (\Cref{tab:official_sampling}); \SD{} models are evaluated with the decoding settings in \Cref{tab:sd_settings}. Full experimental details, prompt formatting, and decoding settings are given in \Cref{app:experimental_details}.

\subsection{\SD{} Improves Code Generation Across the Board}

\definecolor{gaingreen}{HTML}{2E7D32}   %
\definecolor{gainbg}{HTML}{4CAF50}      %
\definecolor{lossred}{HTML}{C62828}     %
\definecolor{lossbg}{HTML}{EF5350}      %

\newcommand{\gcell}[2]{%
  \ifdim #2pt>15pt \cellcolor{gainbg!30}%
  \else\ifdim #2pt>10pt \cellcolor{gainbg!22}%
  \else\ifdim #2pt>5pt \cellcolor{gainbg!14}%
  \else\ifdim #2pt>2pt \cellcolor{gainbg!8}%
  \else \cellcolor{gainbg!3}%
  \fi\fi\fi\fi
  \textbf{#1}$_{\text{\tiny\color{gaingreen}\!+#2}}$%
}
\newcommand{\gcellneg}[2]{\cellcolor{lossbg!12}#1$_{\text{\tiny\color{lossred}\!#2}}$}
\newcommand{\gcellzero}[1]{\cellcolor{gainbg!3}\textbf{#1}}

\begin{table}[t]
\caption{\textbf{\SD{} improves every evaluated model on LiveCodeBench, with the largest gains on medium and hard problems.} Results on LCB\,v6 and LCB\,v5, broken down by difficulty and grouped by reasoning style (thinking vs.\ instruct). Within each model pair, the first row is the base model and the second is +\SD{}; cell shading encodes the change relative to the base row ({\color{gaingreen}green} = improvement, {\color{lossred}red} = decrease, no shading = no change, $\Delta\approx0$).}
\label{tab:main_results_grouped}
\centering
\includegraphics[width=\textwidth]{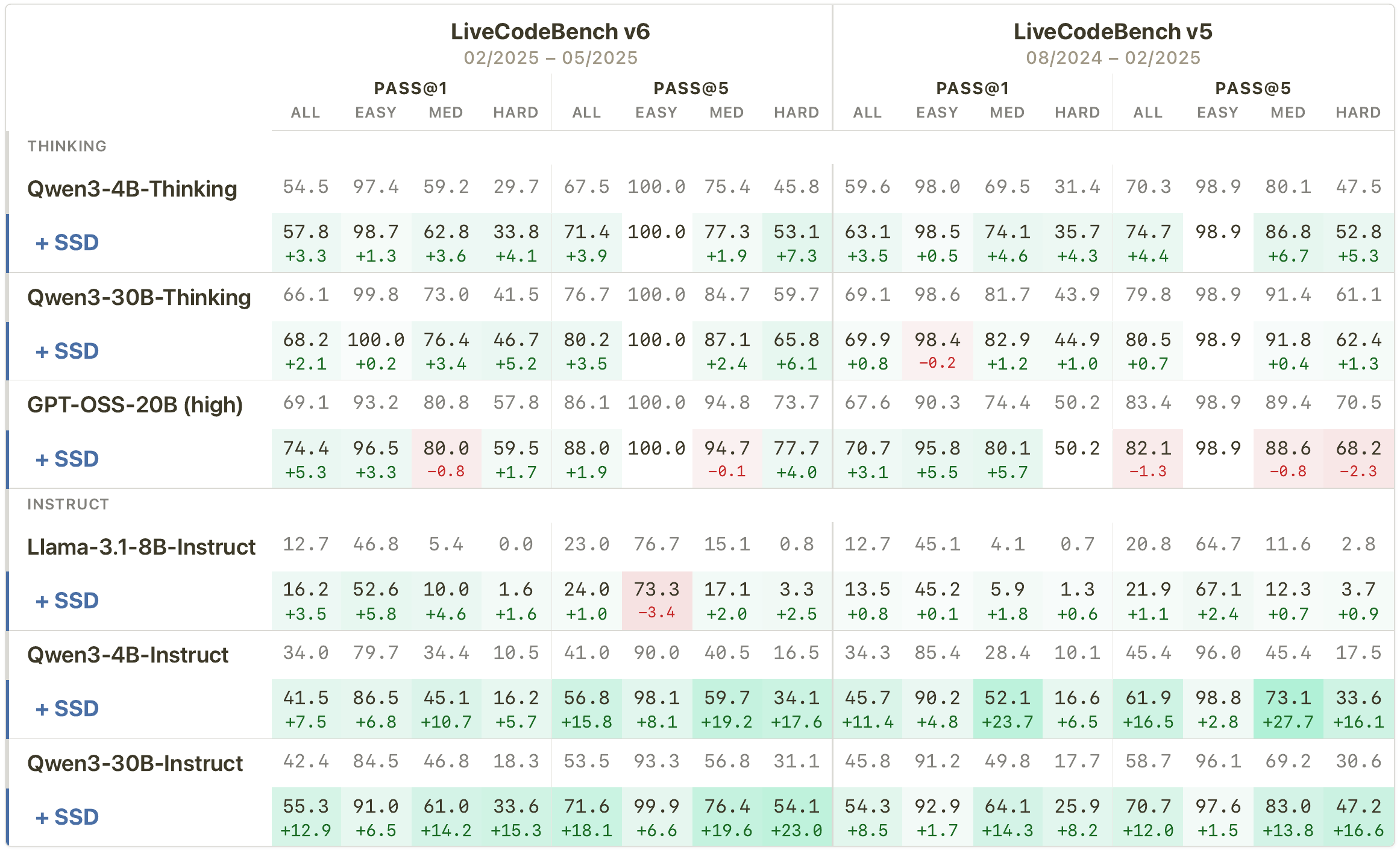}
\end{table}

\textbf{\SD{} yields large gains on LiveCodeBench.}
On LCB\,v6, \SD{} improves Qwen3-30B-Instruct from 42.4\% to 55.3\% pass@1 (+12.9pp, +30.4\% relative).
The gains are broad across the evaluated models: Llama-8B improves by +3.5pp, Qwen3-4B-Instruct by +7.5pp, Qwen3-4B-Thinking by +3.3pp, Qwen3-30B-Thinking by +2.1pp, and GPT-OSS-20B by +5.3pp.
Substantial gains also appear on the larger 374-problem LCB\,v5, where Qwen3-30B-Instruct improves from 45.8\% to 54.3\% pass@1 (+8.5pp).
For a recipe based only on self-generated, unverified solutions and standard supervised fine-tuning, these are consistently strong improvements.

\textbf{\SD{} helps most on harder problems.}
For Qwen3-30B-Instruct on LCB\,v6, pass@1 improves by +6.5pp on easy problems, +14.2pp on medium problems, and +15.3pp on hard problems.
The same concentration appears at pass@5, where the gains are +6.6pp on easy, +19.6pp on medium, and +23.0pp on hard.
A similar pattern is visible in the Qwen thinking models: for Qwen3-4B-Thinking, hard pass@1 improves by +4.1pp versus +1.3pp on easy problems, and for Qwen3-30B-Thinking, the corresponding gains are +5.2pp versus +0.2pp.
Across \Cref{tab:main_results_grouped}, the medium and hard splits consistently account for the largest absolute gains.

\textbf{\SD{} does not collapse diversity.}
A second clear pattern in the table is that the gains are often larger at pass@5 than at pass@1, indicating that \SD{} preserves and even improves generation diversity.
Across the Qwen models on LCB\,v6, the pass@5 gain exceeds the corresponding pass@1 gain: +15.8pp versus +7.5pp for Qwen3-4B-Instruct, +3.9pp versus +3.3pp for Qwen3-4B-Thinking, +18.1pp versus +12.9pp for Qwen3-30B-Instruct, and +3.5pp versus +2.1pp for Qwen3-30B-Thinking.
For Qwen3-30B-Instruct, the same asymmetry is even larger on the hard subset, where pass@5 rises by +23.0pp while pass@1 rises by +15.3pp; on LCB\,v5, pass@5 likewise improves more than pass@1 (+12.0pp versus +8.5pp).
These larger pass@5 gains are consistent with improved diversity across generation samples.

Because \SD{} is trained only on competitive-programming data, a natural concern is that it may hurt performance outside that domain. To test this, we evaluate the same trained model out of the box on benchmarks for math reasoning, general code generation, and code understanding, without any additional training or adaptation. Performance remains broadly stable for the 30B models, see \Cref{sec:transfer}.

\subsection{Global Decoding Policies Cannot Match \SD{}}
\label{sec:decode_only}

\begin{figure}[t]
\centering
\includegraphics[width=\textwidth]{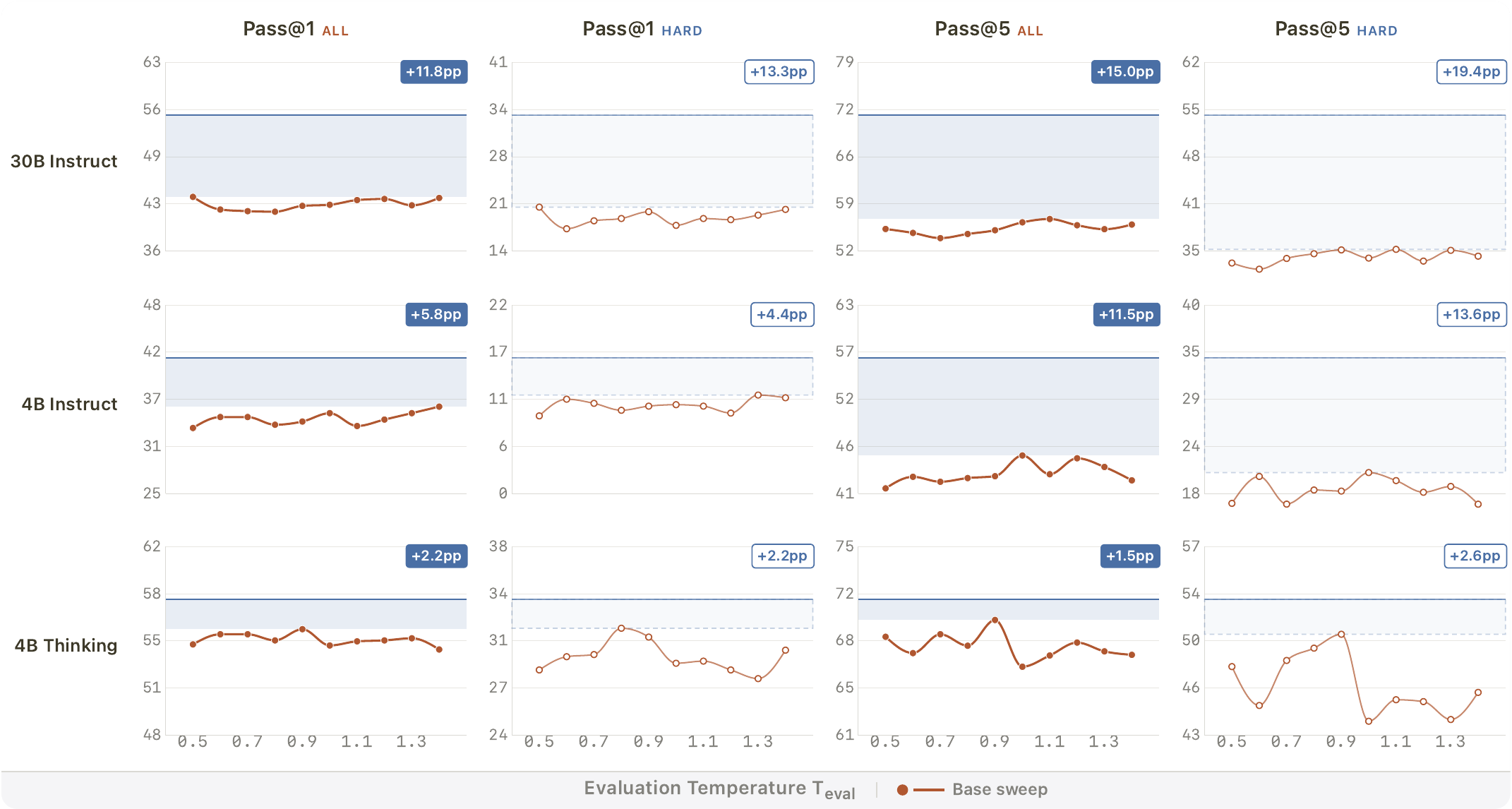}
\caption{\textbf{\SD{} outperforms the best point in the evaluated base-model decoding sweep within standard global decoding policies.}
Each panel shows one model (30B-Instruct, 4B-Instruct, 4B-Thinking) and one metric (pass@1 or pass@5); amber curves sweep base-model evaluation temperature while blue horizontal lines mark \SD{} results from~\Cref{tab:main_results_grouped}. Solid shading marks the margin over all problems; outlined (dashed-border) shading marks the margin on hard problems.
}
\label{fig:ssd_vs_sweep}
\end{figure}

Could the gains in \Cref{tab:main_results_grouped} be recovered \emph{without} training, by tuning only evaluation-time decoding on the original base model? We test this within the standard family of global temperature and truncation $(T_{\textsf{eval}}, \rho_{\textsf{eval}})$ decoding policies, sweeping evaluation-time decoding settings on the base model and comparing the best evaluated base-model configuration to \SD{}. For this comparison, we follow the officially recommended sampling settings for each model (see \Cref{tab:official_sampling}) and sweep temperature extensively to probe the strongest decode-only performance achievable by the base model. \Cref{fig:ssd_vs_sweep} visualizes representative temperature sweeps and the remaining gap to \SD{}.

\textbf{Temperature tuning yields only modest gains on the base model.}
The base-model sweep curves are strikingly flat: for Qwen3-30B-Instruct, pass@1 ranges from 41.3\% to 43.5\% across the evaluated temperatures, a spread of only 2.2\,pp.
The other models show similarly narrow ranges (1.5--3.0\,pp; \Cref{fig:ssd_vs_sweep}).

\textbf{\SD{} still outperforms the best-tuned base model, especially on hard problems and at pass@5.}
Comparing \SD{} against the best-tuned base model, pass@1 advantages remain: +11.8\,pp for Qwen3-30B-Instruct, +5.8\,pp for Qwen3-4B-Instruct, +2.2\,pp for Qwen3-4B-Thinking, and +1.1\,pp for Qwen3-30B-Thinking.
The gap widens on hard problems: for Qwen3-30B-Instruct, \SD{} exceeds the best-tuned base model by +13.3\,pp on hard pass@1 and +19.4\,pp on hard pass@5, both larger than the corresponding all-problem margins.
This pattern holds across all models, the \SD{} advantage is consistently largest on hard problems at pass@5.
These persistent margins indicate that \SD{} produces changes in the model itself in ways no decoding configuration can replicate, an effect we investigate in \Cref{sec:mechanism}.

\subsection{How \SD{} Hyperparameters Interact}
\label{sec:temperature}

To understand the best configuration of training and inference hyperparameters for \SD{}, we performed a grid search over $T_{\textsf{train}}$, and evaluated each checkpoint at multiple $T_{\textsf{eval}}$, with Qwen3-30B-Instruct on LCB\,v6.
We compare two regimes: a no-truncation ablation, where temperature composition is cleanest, and the full truncated setting, where training-time truncation provides an additional gain channel. Full details are in \Cref{app:hyperparameter_sweep}.

\begin{figure}[t]
\centering
\includegraphics[width=\textwidth]{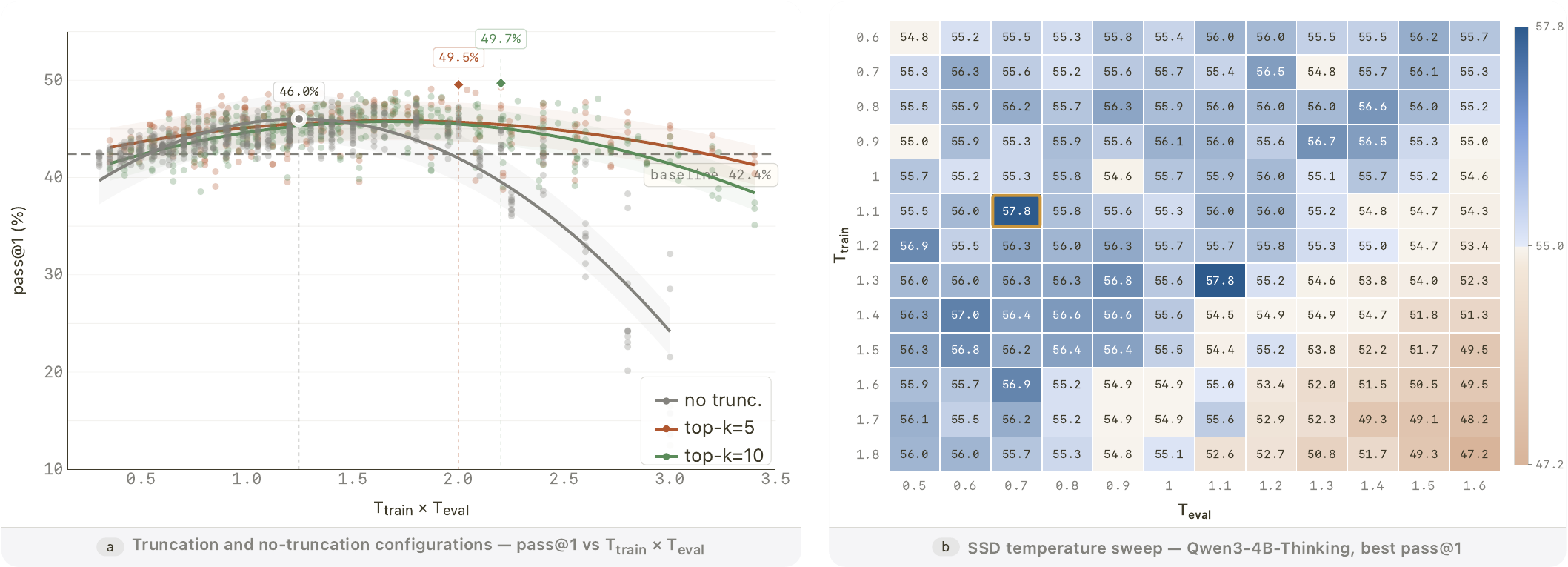}
\caption{\textbf{Training and evaluation temperatures compose through a broad effective-temperature band, while truncation raises the achievable pass@1 within that band.}
\textbf{(a)}~Representative Qwen3-30B-Instruct sweeps on LCB\,v6 against $T_{\textsf{eff}} = T_{\textsf{train}}T_{\textsf{eval}}$: gray = no truncation, amber/green = truncated training-time sampling. Dots are runs, curves are quadratic fits, and the dashed line marks the 42.4\% baseline.
\textbf{(b)}~Qwen3-4B-Thinking on LCB\,v6 with truncation, shown as best pass@1 across iterations over $(T_{\textsf{train}}, T_{\textsf{eval}})$.}
\label{fig:scatter_heatmap_combined}
\end{figure}

\textbf{Without truncation, effective temperature organizes performance.}
$T_{\textsf{train}}$ and $T_{\textsf{eval}}$ trade off each other: $T_{\textsf{train}}$ controls how strongly \SD{} reshapes the model distribution, while $T_{\textsf{eval}}$ controls how aggressively decoding exploits that reshaped distribution. Define $T_{\textsf{eff}} = T_{\textsf{train}} \cdot T_{\textsf{eval}}$; we show that $T_{\textsf{eff}}$ governs performance.
To isolate this, we run a search with only temperature scaling ($T_{\textsf{train}} \in \{0.5, 0.7, 1.0, 1.5, 2.0\}$ and $T_{\textsf{eval}} \in [0.6, 1.5]$; \cpanel{fig:scatter_heatmap_combined}{a},\Cref{fig:heatmap_no_trunc}), and no truncation $\rho_{\textsf{train}}$.
In this regime, the two temperatures compose cleanly: configurations are well governed by $T_{\textsf{eff}}$, with $R^2{=}0.75$ and a quadratic peak near $T_{\textsf{eff}} \approx 1.2$, as formalized in \Cref{app:temp_composition}.
This also explains why higher $T_{\textsf{train}}$ makes the model more responsive to $T_{\textsf{eval}}$: stronger training-time reshaping creates more room for evaluation-time decoding to trade off precision against diversity. Intuitively, higher $T_{\textsf{eff}}$ is preferred to have more diverse generation as long as the generation is not broken.

\textbf{With truncation, the performance ceiling rises.}
When a nontrivial training-time truncation configuration $\rho_{\textsf{train}}$ is used during \SD{} data generation, the truncated runs (amber/green in \cpanel{fig:scatter_heatmap_combined}{a}) remain above the baseline across a wider range of $T_{\textsf{eff}}$ than the no-truncation runs (gray), though the exact collapse onto $T_{\textsf{eff}}$ no longer holds.
This is expected: training-time truncation adds a second improvement channel on top of temperature composition by suppressing low-probability tails during data synthesis.
Among the truncated runs, the best observed setting uses $T_{\textsf{train}}{=}2.0$, $T_{\textsf{eval}}{=}1.1$, and training-time top-$k{=}10$, reaching 49.7\% pass@1 (+7.3\,pp), above all no-truncation results. As expected, the optimal $T_{\textsf{eff}}$ generally shifts towards a higher temperature with more stringent truncation.
Similarly, the diagonal-band pattern appears for Qwen3-4B-Thinking (\cpanel{fig:scatter_heatmap_combined}{b},\Cref{fig:heatmap_thinking}), confirming that the temperature-composition structure extends to thinking models.

\section{Why \SD{} Works}
\label{sec:mechanism}

The gains above raise a natural question: what changes inside the model during \sd{}, and why can't the same effect be achieved by simply adjusting how the original model decodes?
Our hypothesis is that the answer lies in a structural conflict in generation.
Some tokens demand precision, others demand exploration, and any fixed decoding configuration must compromise between them.
\SD{} helps by reshaping token distributions in a way that alleviates this conflict. We validate this mechanism in three steps: a controlled toy simulation, real model analysis, and theoretical decomposition.

\begin{figure}[t]
\centering
\includegraphics[width=\textwidth]{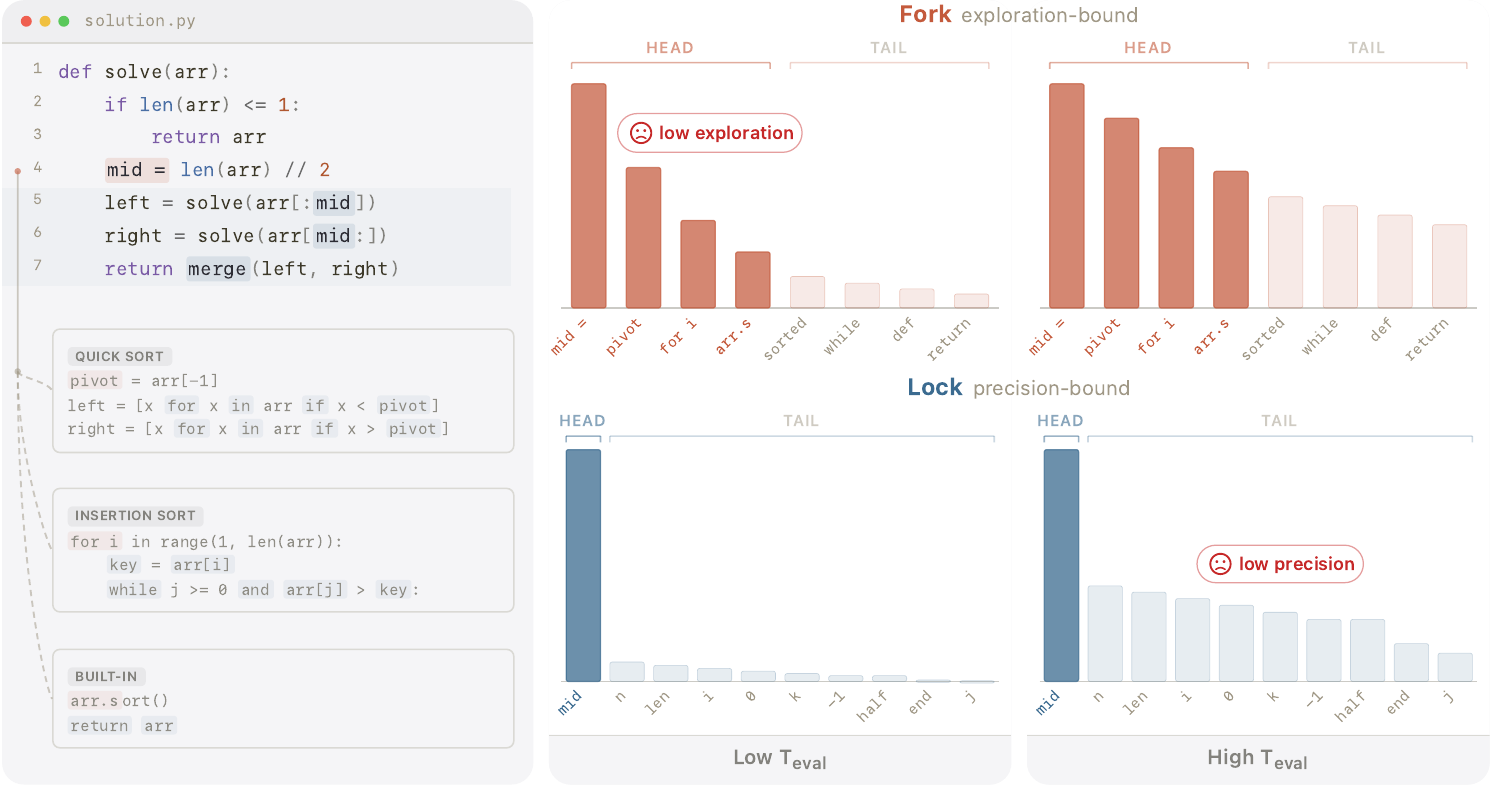}
\caption{\textbf{A single evaluation temperature cannot satisfy both exploration at forks and precision at locks.}
\textbf{Left:} a sorting example in which the algorithm-choice token is a \emph{fork} position (rust-orange), while the later uses of \texttt{mid} are \emph{lock} positions (blue); gray ghost branches indicate other valid algorithms that could have been taken at the fork.
\textbf{Right:} token distributions for the same two context types under low and high $T_{\textsf{eval}}$, with head and tail mass shown explicitly.
Low $T_{\textsf{eval}}$ keeps the lock precise but collapses the fork's viable head (\emph{low exploration}); high $T_{\textsf{eval}}$ restores exploration at the fork but revives the lock's distractor tail (\emph{low precision}).}
\label{fig:tension}
\end{figure}

\subsection{The Precision-Exploration Conflict Hypothesis}
\label{sec:mechanism:conflict}

We will take code generation as an example.
At certain positions, syntax and context leave almost no ambiguity: after \texttt{if n ==}, the model must produce a specific value, and it knows which one, yet a long tail of syntactically plausible alternatives still carries nontrivial probability mass.
At other positions, the distribution is genuinely spread across multiple viable continuations: when beginning the body of a function, the model might open with a \texttt{for} loop, a recursive call, or a data-structure initialization, each leading to a fundamentally different solution.
We hypothesize that these two kinds of positions make fundamentally contradictory demands on the decoding configuration (\Cref{fig:tension}).

We call the first a \textbf{lock}: a position where the distribution is sharply peaked, with very few tokens carrying most of the mass and a long distractor tail carrying the rest.
We call the second a \textbf{fork}~\citep{bigelow2025forking,wang2025beyond8020}: a position where the distribution is spread across multiple plausible tokens that can lead to meaningfully different downstream continuations.
Locks demand precision: commit to the dominant token and suppress the tail.
Forks demand exploration: spread mass across viable alternatives to avoid missing the good paths.

Under this view, inference temperature $T_{\textsf{eval}}$ is what makes the conflict irreconcilable.
Scaling by $T_{\textsf{eval}}$ flattens or sharpens the entire distribution $p_T(v) \propto p(v)^{1/T}$: higher $T_{\textsf{eval}}$ compresses probability gaps, pulling tokens toward equal footing; lower $T_{\textsf{eval}}$ widens them, amplifying the dominant peak. There is a dilemma.
Lowering temperature sharpens the peak at a lock, suppressing distractors, but starves a fork of the diversity it needs.
Raising temperature diversifies the head at a fork, giving lower-ranked correct continuations a chance, but destabilizes locks as the distractor tail regains mass.
The best global setting, applied to every context in the sequence, is therefore necessarily a compromise: the temperature that helps forks is precisely what lets distractors resurface at locks.

If this picture is right, then \SD{} should not sharpen the model uniformly: it should suppress distractor tails at locks while leaving more useful room for exploration at forks.

\subsection{How \SD{} Reshapes a Model: Toy Simulation and Real-Model Analysis}
\label{sec:mechanism:toy}
\label{sec:mechanism:training}
\label{sec:mechanism:decoding}

We now test that prediction in two settings. We begin with a toy environment where the conflict is explicit and success can be computed exactly. We then ask whether the same qualitative pattern appears in a real model.

\paragraph{Controlled simulation.}
We begin with a minimal environment that contains exactly the structure in our hypothesis. Successful trajectories must pass through one fork state and then three lock states before reaching PASS; any wrong decision leads to FAIL. At the fork, several continuations are genuinely plausible. At each lock, one token is correct but a distractor tail remains. Because every transition is specified explicitly, the probability of success can be computed in closed form for any decoding temperature (full details are given in \Cref{app:simulation_details}).

Even in this minimal setting, the same dilemma appears. Sweeping a single global decoding temperature on the base model recovers the same tradeoff as in \Cref{sec:mechanism:conflict}: colder decoding protects the locks but starves the fork, while hotter decoding helps the fork but breaks the locks (\Cref{fig:v19_storyboard}). The base model therefore operates at a narrow compromise.

In the toy, \SD{} changes that compromise by reshaping the two regimes differently (\Cref{fig:spike_plateau}). At lock-like states, the low-probability tail is stripped away, so the dominant token becomes much harder to dislodge. At fork-like states, several plausible continuations remain near the top, but the useless tail is reduced and the surviving options become more even. These local changes widen the viable decoding regime itself: after \SD{}, the best decoding temperature shifts much higher, and success probability rises substantially.

\begin{figure}[t]
\centering
\includegraphics[width=\textwidth]{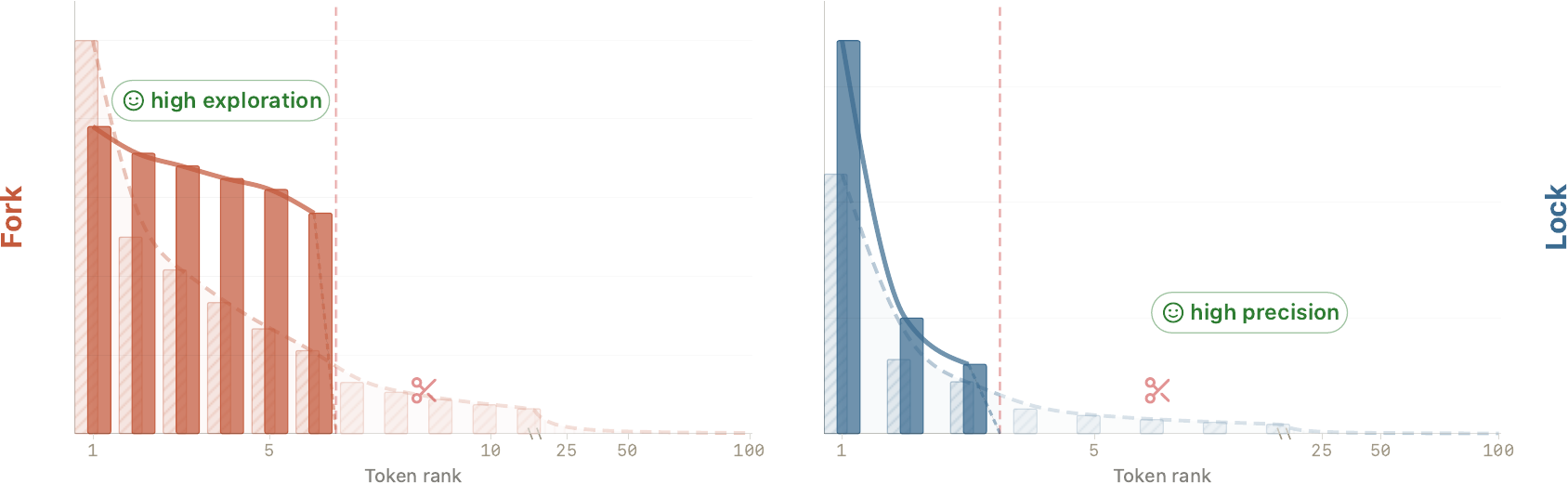}
\caption{\textbf{\SD{} turns forks into plateaus and locks into spikes.}
Tokens are ranked by probability. Hatched bars and dashed curves show the base model; solid bars and solid curves show the model after \SD{}; the red dashed cutoff marks the support retained during \SD{}.
\textbf{(a)}~Fork-like state: the diffuse tail is trimmed, but several top continuations remain and become more evenly weighted, forming a broad plateau over viable branches.
\textbf{(b)}~Lock-like state: the same rule trims the tail much more aggressively and concentrates mass on the dominant token, producing a sharper spike.}
\label{fig:spike_plateau}
\end{figure}

\paragraph{The synergy between training and decoding.}
The toy also shows why training and decoding are complementary rather than interchangeable. Training does not solve the fork by itself; it makes the locks less fragile. Decode-only temperature tuning does not clean up the locks by itself; it spends precision before it gains enough exploration at the fork. The improvement comes only when both stages act together. Training changes the distribution so the locks are safer; decoding then uses that extra room to explore the fork.

\paragraph{Real-model evidence.}
We now look for the same pattern from the base Qwen3-30B-Instruct model and its \SD{} counterpart on LCB\,v6. The same two signatures appear. Relative to the base model, \SD{} reaches decoding with a cleaner head and a weaker distractor tail.

\Cpanel{fig:empirical_triptych}{a} shows the first effect directly. When tokens are ordered by probability, cumulative mass rises more quickly for \SD{} through the top ranks. Less probability is left behind in diffuse distractor tails before decoding even begins. This is the real-model analogue of the lock side of the toy.

\Cpanel{fig:empirical_triptych}{b--d} show the second effect. Under the same evaluation-time decoding temperature and truncation $(T_{\textsf{eval}}, \rho_{\textsf{eval}})$, raising $T_{\textsf{eval}}$ changes the base model much less: the surviving set stays close to a singleton, so temperature has limited leverage. \SD{} behaves differently. As temperature rises, several top continuations remain viable, and the probabilities among those surviving options spread out much more strongly. This advantage persists even when the two models place similar probability mass in their top 20 tokens. The real-model evidence therefore matches the toy: \SD{} removes distractor mass where commitment matters and enlarges the region in which temperature can be used for exploration.

\begin{figure}[t]
\centering
\includegraphics[width=\textwidth]{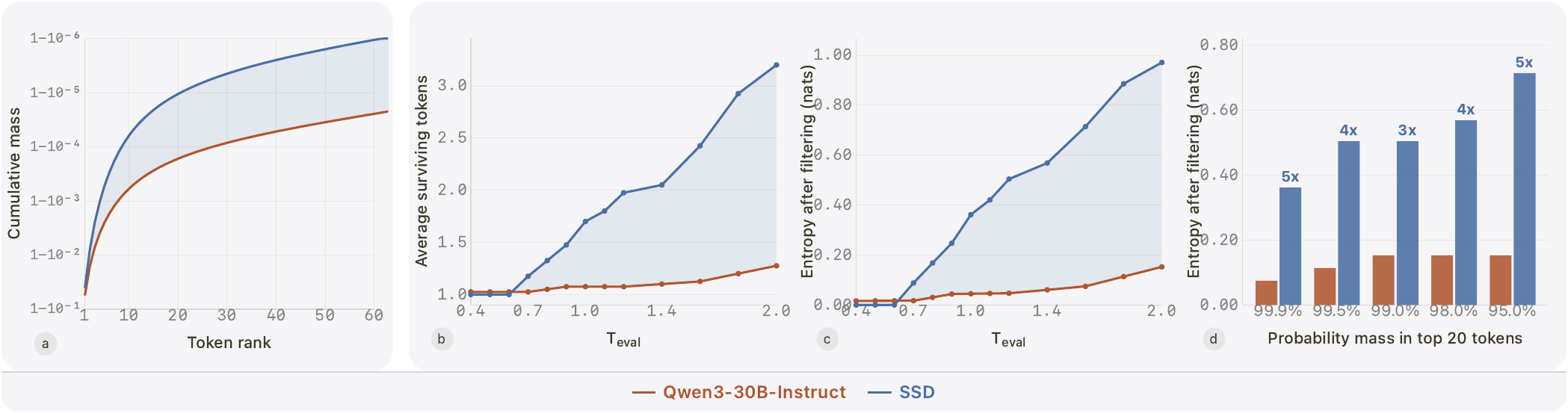}
\caption{\textbf{Real-model evidence that \SD{} both compresses distractor tails and makes $T_{\textsf{eval}}$ more effective near the head.}
Amber: base Qwen3-30B-Instruct; blue: after \SD{}.
\textbf{(a)}~When tokens are sorted by model probability, cumulative mass rises faster for \SD{}, indicating a cleaner head and weaker diffuse tail.
\textbf{(b)}~As $T_{\textsf{eval}}$ increases, more tokens survive truncation in \SD{} than in the base model.
\textbf{(c)}~The entropy of the distribution after truncation increases much more strongly for \SD{}.
\textbf{(d)}~This higher entropy after truncation persists even when the two models place similar probability mass in their top 20 tokens, providing more viable alternatives for evaluation time exploration.
Together, the base model enters decoding with more tail mass, while \SD{} offers more usable room for temperature to diversify the top of the distribution.}
\label{fig:empirical_triptych}
\label{fig:coverage}
\label{fig:pd_temperature}
\end{figure}

The toy isolates the mechanism, and the real-model analysis shows the same mechanism in practice. \SD{} does not remove the conflict by making every context uniformly sharper. It relaxes the conflict asymmetrically: forks retain more usable alternatives near the top of the distribution, while locks become safer. That is why higher-temperature decoding becomes newly effective after training. The next subsection formalizes why these two changes can coexist: reduced tail mass where precision matters and more usable diversity near the top where exploration matters.

\subsection{A Theoretical View of \SD{}}
\label{sec:mechanism:objective_alt}

We now turn to the theoretical view behind that picture.
\SD{} fits the distribution induced by sampling the base model with $T_{\textsf{train}}$ and $\rho_{\textsf{train}}$. That shift in the training signal leads to the objective decomposition below, explains why forks and locks respond differently, clarifies the entropy picture, and also explains why decode-only tuning cannot reproduce the same effect (\Cref{app:theory_new:setup,app:theory_new:signal,app:theory_new:locks_forks,app:temp_composition,app:theory_new:entropy,app:theory_new:decode_only}).

\paragraph{\SD{} induces support compression and within-support reshaping.}
We begin with the distribution that \SD{} fits. During data synthesis, we sample from the base model under $(T_{\textsf{train}}, \rho_{\textsf{train}})$. At any context, this procedure produces a retained set $S$ of tokens that survive temperature scaling and truncation, together with a renormalized distribution $q$ over that set. Let $\textsf{KeptMass}_\theta$ denote the probability mass that the model under optimization assigns to $S$, and write $T \equiv T_{\textsf{train}}$. With this notation, the induced loss can be written as

\begin{equation}
\label{eq:three_term_alt}
\mathcal{L}(\theta)
\;=\;
\underbrace{-\log \textsf{KeptMass}_\theta}_{\textup{support compression (via $\rho_{\textsf{train}}$)}}
\;+\;
\underbrace{(1-T)\,H_{1/T}\!\bigl(p_\theta(\cdot \mid S)\bigr)}_{\textup{within-support reshaping (via $T_{\textsf{train}}$)}}
\;+\;
\underbrace{T \cdot \mathrm{KL}\!\bigl(q \,\|\, p_{\theta,T}(\cdot \mid S)\bigr)}_{\textup{alignment to the base model}}
\;+\;
\mathrm{const},
\end{equation}

Here $H_{1/T}(\pi)$ is the R\'enyi entropy of order $1/T$, and $p_{\theta,T}(\cdot \mid S)$ is the model's tempered distribution restricted to $S$. The three terms have clear roles: the first term drives support compression, which removes diffuse tail mass to concentrate probability on a smaller set of viable tokens, the second reshapes the head within that set, and the third keeps that reshaping aligned with the base model on that same set (\Cref{app:theory_new:setup,app:theory_new:signal}). This decomposition is central because it shows simple self distillation is not mere imitation; it enforces both \textit{support compression} and \textit{head reshaping}.

\paragraph{\SD{} sharpens locks while preserving forks.}
Once written this way, the lock/fork asymmetry follows from what survives into the retained set at each type of context. At a lock, only one or a few tokens survive truncation, so support compression dominates: distractor mass is pushed out of the tail and the surviving head becomes relatively insensitive to $T_{\textsf{eval}}$. At a fork, several plausible continuations survive, so within-support reshaping has room to flatten and preserve the head without reopening the discarded tail. Appendix \Cref{app:theory_new:locks_forks,app:temp_composition} formalizes this asymmetry and shows how training-time and evaluation-time temperatures compose inside the retained set.

\vspace{-4pt}
\paragraph{\SD{} lowers total entropy while preserving head exploration.}
The fine-tuned model can become globally sharper while remaining more explorable at evaluation time because total entropy and useful exploration concern different objects. Appendix \Cref{app:theory_new:entropy} decomposes full-vocabulary entropy into a gate term, a head term, and a tail term: \SD{} lowers the gate and tail contributions by concentrating mass on the retained set, while the conditional head can remain broad enough at fork-like contexts for $T_{\textsf{eval}}$ to diversify among viable continuations.

\vspace{-4pt}
\paragraph{Understanding why decode-only tuning cannot match \SD{}.}
Appendix \Cref{app:theory_new:decode_only} shows that decode-only $(T_{\textsf{eval}}, \rho_{\textsf{eval}})$ policies remain constrained by the base model's existing ranking and cumulative curves: they can reweight a fixed distribution, but they cannot steepen locks and clean up fork heads in a context-dependent way. \SD{} changes the distribution itself, which is why the empirical decode-only gap in \Cref{sec:decode_only} persists.

\subsection{A Surprising Case: Bad Data, Good Results}
\label{sec:high_temp}

\begin{wrapfigure}[30]{r}{0.34\textwidth}
\vspace{-14pt}
\centering
\includegraphics[width=\linewidth]{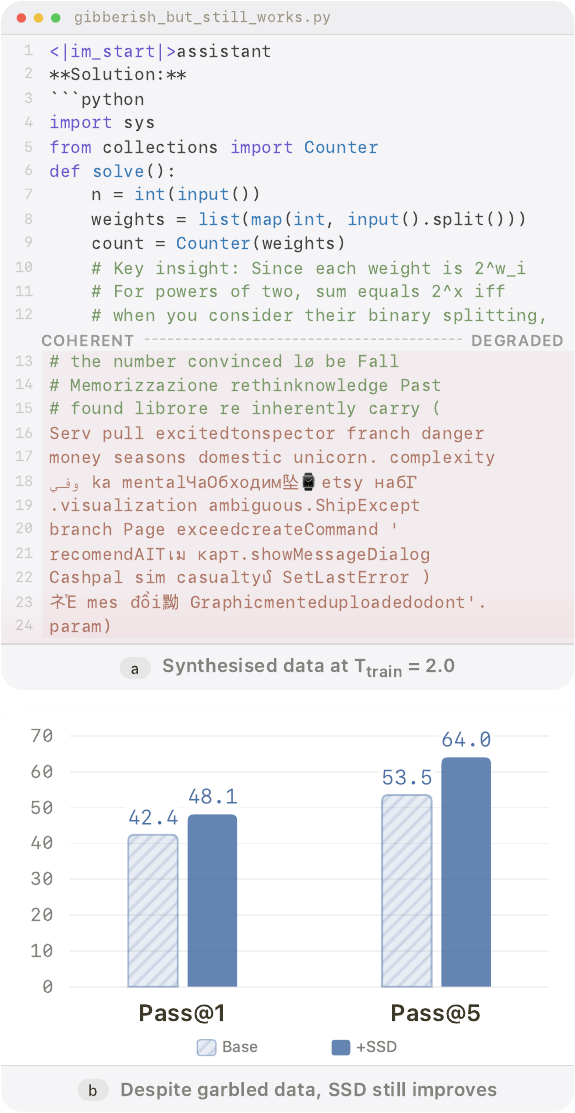}
\vspace{-8pt}
\captionsetup{font=footnotesize}
\caption{\textbf{Bad data, good results.}
\textbf{(a)}~At $T_{\textsf{train}}{=}2.0$ without truncation, a representative sample degrades into gibberish; ${\sim}62\%$ of outputs contain no extractable code.
\textbf{(b)}~The fine-tuned model still surpasses the 42.4\%/53.5\% base-model pass@1/pass@5, reaching 48.1\% and 64.0\%.}
\label{fig:high_temp}
\end{wrapfigure}

We now push \SD{} into an intentionally pathological regime as a stress test for our established hypothesis and understanding, that \SD{} makes high-temperature $T_{\textsf{eff}}$ possible and beneficial, as well as that training and decoding are complementary to each other in \SD{}.
Starting from Qwen3-30B-Instruct, we raise the training temperature to $T_{\textsf{train}}{=}2.0$ and disable truncation entirely (setting $\rho_{\textsf{train}}$ to be vacuous), asking whether \SD{} still helps when the sampled training outputs are overwhelmingly poor as programs.
If the benefit of \SD{} depended primarily on training on good solutions, this setting should be close to a failure case. More details are in \Cref{app:high_temp_details}.

\paragraph{In this stress test, the synthesized data is almost gibberish.}
Without truncation to suppress the tail, sampling at $T_{\textsf{train}}{=}2.0$ produces outputs that are often unusable as code.
About ${\sim}62\%$ contain no extractable code at all, and even seemingly coherent solutions frequently devolve into multilingual gibberish mid-sequence (\cpanel{fig:high_temp}{a}).
By ordinary data-quality standards, this is unusable as training data for SFT.

\paragraph{\SD{} still improves the model materially.}
Even when the synthesized outputs devolve into gibberish, the resulting fine-tuned model is not merely salvageable, it improves substantially.
\SD{} improves the model to 48.1\% pass@1 and 64.0\% pass@5, for gains of +5.7\,pp and +10.5\,pp respectively (\cpanel{fig:high_temp}{b}).
This peak is not an isolated lucky cell: it sits inside a contiguous late-training ridge around $T_{\textsf{eval}} \in [0.8, 1.1]$, with several neighboring checkpoint and temperature pairs remaining within about 1 to 2\,pp of the optimum (\cpanel{fig:high_temp_expanded}{b}).
As in earlier sections, the improvements are concentrated on hard problems: at the best setting, hard pass@1 increases by +7.3\,pp and hard pass@5 by +13.8\,pp.
This demonstrates that \textit{support compression} and \textit{distribution reshaping} extract useful learning signals regarding token quality, meaning program correctness might not mainly drive the gains.

\paragraph{The gain depends on evaluation-time truncation.}
\Cpanel{fig:high_temp_expanded}{b} shows a clear bounded operating region: within the viable low-$T_{\textsf{eval}}$ band, the best results form a late-training ridge rather than a single spike, but performance still degrades sharply once $T_{\textsf{eval}}$ becomes too high.
That pattern is consistent with the idea that training alone is not enough here: without truncation during training, diffuse distractor tails remain and must be cleaned up at evaluation time by $\rho_{\textsf{eval}}$.
This is also why the gains remain smaller than those of the headline truncated setting.
Taken together, the case study suggests that \SD{} is not drawing its benefit mainly from training on correct code.
Even in this pathological regime, the useful signal still comes from how high-temperature sampling reshapes token probabilities, while decoding-time truncation recovers enough precision to make that reshaping useful.

\section{Related Work}
\label{sec:related}

\paragraph{Self-training and self-distillation.}
Learning from model-generated targets has long been studied in self-training and distillation, including classical self-training, knowledge distillation, sequence-level distillation, and self-distillation \citep{amini2022self-training, he2020revisiting, hinton2015distilling, kim2016sequence, furlanello2018born}. In language modeling, recent work extends this paradigm to on-policy distillation and related self-distillation variants that supplement self-generated sequences with privileged information, textual or verbal feedback, additional context, or interaction signals \citep{agarwal2024onpolicy, zhao2026selfdistilled, hubotter2026sdpo, song2026rltf, xiong2026ovd, penaloza2026pidistill, ye2026opcd, shenfeld2026sdft, buening2026userinteractions, stein2026gates}. In contrast, \SD{} uses only temperature-shifted samples from the base model and standard cross-entropy training, without privileged context, feedback-conditioned teachers, or auxiliary supervision.

\paragraph{Code generation and synthetic data.}
In code generation, synthetic-data pipelines often rely on large-scale sampling followed by filtering, clustering, verification, or execution feedback \citep{li2022alphacode, le2022coderl, liu2025rstarcoder}. Related self-training approaches such as STaR and ReST$^{EM}$ likewise convert self-generated outputs into supervision through correctness-based filtering or external feedback \citep{zelikman2022star, singh2024beyond}. \SD{} differs in that it trains directly on raw, unverified model outputs.

\paragraph{Reasoning and RL for math and coding.}
Recent progress on reasoning and code generation has come from chain-of-thought prompting, zero-shot reasoning prompts, self-consistent sampling, self-bootstrapping, and RL-based post-training for math and code \citep{wei2022chainofthought, kojima2022zeroshot, wang2023selfconsistency, zelikman2022star, shao2024deepseekmath, guo2025deepseekr1, openai2025competitive}. A complementary line of work studies reasoning improvement at the token level, identifying critical, high-entropy, or forking tokens as disproportionately important decision points in reasoning and RL trajectories \citep{bigelow2025forking, lin2024critical, vassoyan2025ignore, wang2025beyond8020, cheng2025reasoningexploration, wang2025hierarchicalreasoning, gandhi2025cognitive, li2025llms, chu2025sftmemorizesrlgeneralizes}. Our focus is different: rather than asking which tokens an RL algorithm should emphasize, we ask how far plain cross-entropy training on a model's own raw outputs can go without rewards or verifiers, and why it reshapes the distribution in a way that decode-only tuning cannot match.

\paragraph{Decoding and truncation.}
At inference time, top-$k$ sampling, nucleus sampling, and truncation-as-desmoothing analyze how temperature and support restriction shape generation quality \citep{fan2018hierarchical, holtzman2020curious, hewitt2022truncation}. Our contribution is not a new decoding rule. Instead, we show that training on samples generated under shifted decoding can alter the model itself, making a simple fixed decoding policy substantially more effective at test time.

\paragraph{Self-improvement without external reward.}
Several methods improve language models using self-generated signal without human labels, but they still rely on internal critique, judging, filtering, or iterative self-evaluation \citep{wang2023selfinstruct, bai2022constitutionalai, huang2023selfimprove, yuan2024selfrewarding}. A closely related line, often framed as unsupervised RLVR or intrinsic-signal learning, replaces ground-truth rewards with internal signals such as majority vote, entropy, confidence, or self-certainty \citep{he2026far, zuo2025ttrl, agarwal2025entropy, prabhudesai2025rent, zhao2025learningreason, zhang2025nofreelunch}; related analyses also study entropy reduction as a driver of reasoning gains and entropy collapse as a limit on exploration during RL \citep{cui2025entropy}. \SD{} differs from this line in both method and mechanism. It is not an RL procedure that directly optimizes a scalar entropy objective or uniformly drives policy entropy downward. Instead, training on temperature-shifted, truncated self-samples reshapes the token distribution in a context-dependent way: it suppresses diffuse tail mass while preserving, and at fork-like contexts even increasing, useful entropy within the retained head. As a result, the model can become lower-entropy overall while more explorable where it matters. In this sense, \SD{} is better understood as support compression plus within-support reshaping, rather than direct Shannon-entropy minimization \citep{renyi1961measures}.

\section{Conclusion}
\label{sec:conclusion}

We have shown that a model can improve code generation by training on its own raw outputs alone. Across six models, \sd{} consistently improves LiveCodeBench, with the largest gains on harder problems; for Qwen3-30B-Instruct, pass@1 rises from 42.4\% to 55.3\% on LiveCodeBench\,v6. Our evidence points to a simple explanation: code generation mixes precision-bound locks and exploration-bound forks, and \SD{} reshapes token distributions so decoding can explore useful branches without reopening distractor tails. More broadly, these results suggest that strong code models contain latent capability that can be unlocked without a verifier, a teacher, or reinforcement learning.

\section*{Acknowledgments}

We thank David Grangier, Tatiana Likhomanenko, Zijin Gu, Samy Bengio, Vivek Rathod, Josh Susskind, Shuangfei Zhai, and Jiatao Gu for stimulating discussions and valuable suggestions during the preparation of this manuscript.

\clearpage

\appendix
\addtocontents{toc}{\protect\setcounter{tocdepth}{2}}

\etocsettocstyle{}{}
\etocsettocdepth{subsection}

\etocsetstyle{section}
  {}{}
  {\normalsize\noindent
   \makebox[1.5em][l]{\sffamily\bfseries\color{bscolor}\etocnumber}%
   \sffamily\etocname
   \nobreak\leaders\hbox to 0.7em{\hss.\hss}\hfill\nobreak
   \makebox[1.5em][r]{\sffamily\etocpage}\par
   \vskip 3pt}
  {}

\etocsetstyle{subsection}
  {}{}
  {\normalsize\noindent\hspace{1.5em}%
   \makebox[2.5em][l]{\sffamily\small\color{bscolor}\etocnumber}%
   {\sffamily\small\etocname}%
   \nobreak\leaders\hbox to 0.7em{\hss.\hss}\hfill\nobreak
   \makebox[1.5em][r]{\sffamily\small\etocpage}\par
   \vskip 1.5pt}
  {}

\begin{tcolorbox}[enhanced, breakable, frame hidden,
  colback=titlebg, arc=4pt,
  left=5mm, right=5mm, top=4mm, bottom=4mm,
  before skip=8pt plus 2pt, after skip=12pt plus 2pt]
{\Large\sffamily\bfseries Appendix Contents}
\vskip 6pt
{\color{bordercolor}\hrule height 0.4pt}
\vskip 8pt
\tableofcontents
\end{tcolorbox}

\section{Decoding Pipeline: From Notation to Implementation}
\label{app:decoding_pipeline}

All inference in this paper uses vLLM~v0.11.0~\citep{kwon2023vllm} (commit \texttt{b8b302c}).\footnote{\url{https://github.com/vllm-project/vllm/tree/b8b302cde434df8c9289a2b465406b47ebab1c2d}}
The decoding operator used in the main text (\Cref{sec:method}) maps directly to vLLM's \texttt{Sampler} class and its associated helpers. In this appendix, we unpack it into explicit temperature, top-$k$, and top-$p$ steps.
This section documents the exact order of operations to make the decoding semantics fully reproducible.
Other logit processors available in vLLM (repetition, frequency, and presence penalties;
\texttt{min\_p}; logit bias) are not used in our experiments; see the
\texttt{Sampler.forward} docstring in \texttt{v1/sample/sampler.py} for the complete
ordering when those processors are active.
Given raw logits $z_v$ from the language-model head, the pipeline applies four steps in the order described below; \Cref{fig:decode_pipeline} shows the full pipeline with annotated source excerpts.

\paragraph{Step~1: Temperature scaling.}
Temperature is applied first, before any truncation, by dividing every logit by~$T$ in place (\Cref{fig:decode_pipeline}, panel~1).
After a subsequent softmax, this is equivalent to raising each probability to the power~$1/T$ and renormalizing, i.e.\ the $\textsf{Temper}_T$ operator defined in \Cref{eq:theory_new:temper}.
Temperatures below $10^{-5}$ trigger greedy (argmax) decoding, bypassing all subsequent steps.

\paragraph{Step~2: Top-$k$ filtering.}
Logits are sorted in ascending order (\Cref{fig:decode_pipeline}, panel~2). The $k$-th largest value is found via
\texttt{gather}, and all logits strictly below it are set to~$-\infty$.
This operates on the already temperature-scaled logits, so the ranking reflects
$z_v/T$ rather than the raw logits.
When \texttt{top\_k\,=\,0} or \texttt{top\_k\,$\ge$\,|V|}, this step is skipped entirely.

\paragraph{Step~3: Top-$p$ (nucleus) filtering.}
On the same sorted tensor, a softmax is computed over the top-$k$ survivors (\Cref{fig:decode_pipeline}, panel~3).
A cumulative sum ascending from the smallest probability identifies the
lowest-mass tokens whose removal still leaves cumulative mass $\ge \text{top-}p$.
At least one token (the highest-probability one) is always retained.
The result is then scattered back to the original vocabulary order.

\paragraph{Step~4: Sampling via the Gumbel-max trick.}
Rather than calling \texttt{torch.multinomial} (which incurs synchronization
between CPU and GPU), vLLM draws independent $\mathrm{Exp}(1)$ noise, divides the
post-truncation probabilities by that noise, and takes the argmax (\Cref{fig:decode_pipeline}, panel~4).
This is mathematically equivalent to multinomial sampling from the surviving
support.

\paragraph{Correspondence to paper notation.}
The four steps above implement exactly the retained-support definition
used throughout the theory (\Cref{app:theory_new:setup}):
\[
S_s
\;=\;
\textsf{TopP}\!\Bigl(
\textsf{Temper}_{T}[p_0(\cdot \mid s)]
\Big|_{\textsf{TopK}(\textsf{Temper}_{T}[p_0(\cdot \mid s)],\,k)},\;
\text{top-}p
\Bigr).
\]
The implementation confirms that the order is \emph{temper\,$\to$\,top-$k$\,$\to$\,top-$p$\,$\to$\,sample}: temperature scaling is applied to logits first, top-$k$ filters on the tempered distribution, and top-$p$ operates within the top-$k$ retained set (renormalized via softmax over survivors before computing the cumulative threshold).

\begin{figure}[p]
\centering
\includegraphics[width=\textwidth]{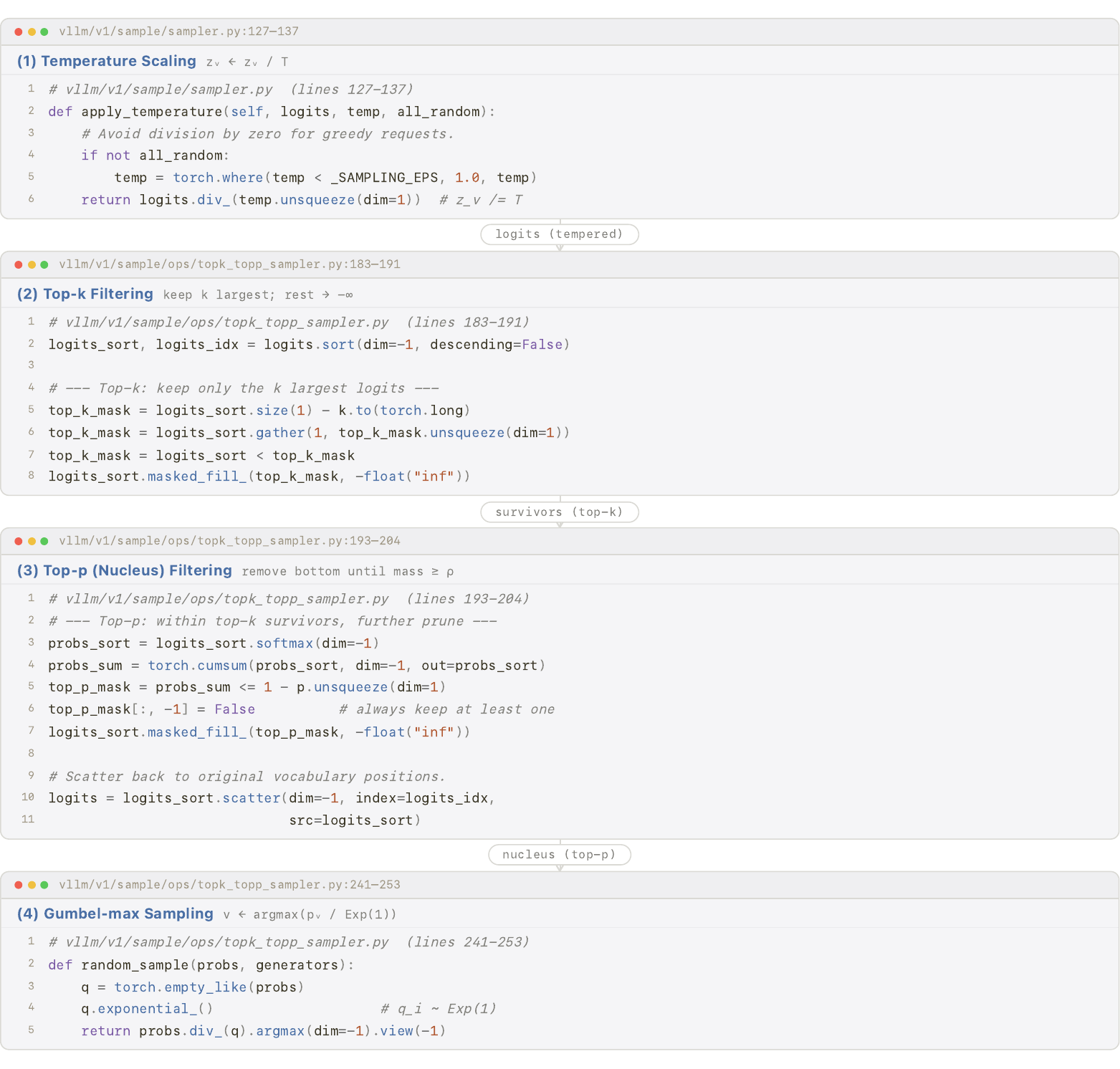}
\caption{\textbf{The vLLM v0.11.0 decoding pipeline used throughout this paper.}
The pipeline applies four steps in sequence: (1)~temperature scaling divides all logits by~$T$; (2)~top-$k$ filtering keeps the $k$ largest tempered logits and sets the rest to $-\infty$; (3)~top-$p$ filtering further prunes from the bottom of the surviving set until cumulative mass reaches the chosen top-$p$ threshold; and (4)~Gumbel-max sampling draws a token from the resulting distribution without synchronization between CPU and GPU.
Each panel shows the corresponding vLLM source excerpt with file path and line numbers.}
\label{fig:decode_pipeline}
\end{figure}

\section{A Theoretical View of \SD{}: Full Analysis}
\label{app:theory_new}
\label{sec:mechanism:objective}

This appendix formalizes the mechanism behind \SD{} in a minimal setting. We proceed in the same order as the mechanism story in the main text. We first define the local objects that \SD{} fits at a single decoding context. We then analyze the training objective induced by those objects and show why \SD{} has a nontrivial learning signal even though it trains on the model's own outputs. Next, we explain why the same global objective suppresses diffuse lock tails while preserving useful exploration at fork-like contexts. We then show why the resulting student can become lower-entropy overall while still preserving the conditional head entropy needed for exploration. Finally, we explain why decode-only tuning on the frozen model cannot reproduce the same effect.

\subsection{Notation and Setup}
\label{app:theory_new:setup}

We analyze \SD{} at the level of a single decoding context. At such a context, the frozen model's training-time decoding policy first selects a retained support and then induces a target distribution on that support; ordinary cross-entropy is then used to fit that target. We therefore begin by defining these objects in the order they are used.

Throughout this appendix, we use the word \emph{teacher} only as shorthand for the frozen pre-\SD{} model that generates the \SD{} training data. There is no separate or stronger external teacher model. The distinction is purely temporal: the teacher is the frozen model before \SD{} training, while the student is the model being optimized to fit the teacher-induced target.

A \emph{context} is a pair $s=(x,y_{<t})$ consisting of a prompt $x$ and the tokens generated so far. At each context $s$, a model with parameters $\theta$ defines a next-token distribution $p_\theta(\cdot \mid s)\in\Delta^{V-1}$ over the vocabulary $\mathcal{V}$. We write $p_0(\cdot \mid s)$ for the frozen pre-\SD{} model used to synthesize the training data, and $p_\theta(\cdot \mid s)$ for the student distribution learned by supervised fine-tuning. At initialization, the student is the original model, so $p_\theta=p_0$.

Training-time decoding follows the same basic pipeline used in the main paper: first apply temperature, then truncate, then renormalize. For temperature $T>0$ and a nonempty set $S \subseteq \mathcal{V}$, define the tempered distribution on $S$ by
\begin{equation}
\label{eq:theory_new:temper}
\textsf{Temper}_T^{S}[p](v)
\;=\;
\frac{p(v)^{1/T}\,\mathbf{1}\{v \in S\}}
     {\sum_{u \in S} p(u)^{1/T}}.
\end{equation}
When $S=\mathcal{V}$, we simply write $\textsf{Temper}_T[p]$. Low temperature sharpens a distribution, while high temperature flattens it.

Given a distribution $\pi$ on $\mathcal{V}$, $\textsf{TopK}(\pi,k)$ returns the set of the $k$ tokens with largest probability under $\pi$, with ties broken by a fixed deterministic rule. For any nonempty set $K \subseteq \mathcal{V}$ with $\pi(K)>0$, write
\[
\pi(v \mid K)
\;=\;
\frac{\pi(v)\,\mathbf{1}\{v \in K\}}{\sum_{u \in K}\pi(u)}.
\]
Given such a set $K$ and a top-$p$ threshold $\text{top-}p \in (0,1]$, $\textsf{TopP}(\pi|_K,\text{top-}p)$ returns the smallest prefix of the ranking on $K$, sorted by decreasing $\pi(\cdot \mid K)$-mass, whose cumulative mass under $\pi(\cdot \mid K)$ is at least $\text{top-}p$. Throughout the paper, both training-time and evaluation-time decoding first apply temperature and then truncate and renormalize.

Fix training-time decoding parameters $(T_{\textsf{train}}, k_{\textsf{train}}, \text{top-}p_{\textsf{train}})$. At context $s$, let $S_s$ be the teacher's retained support after applying training-time temperature, top-$k$, and top-$p$:
\begin{equation}
\label{eq:theory_new:retained_set}
S_s
\;\equiv\;
\textsf{TopP}\!\Bigl(
\textsf{Temper}_{T_{\textsf{train}}}[p_0(\cdot \mid s)]
\Big|_{\textsf{TopK}(\textsf{Temper}_{T_{\textsf{train}}}[p_0(\cdot \mid s)],\,k_{\textsf{train}})},
\text{top-}p_{\textsf{train}}
\Bigr).
\end{equation}
In words, $S_s$ is the teacher's retained support at context $s$: the set of tokens that survive training-time temperature scaling and truncation.

The target fitted by \SD{} at context $s$ is the truncated and renormalized tempered teacher distribution
\begin{equation}
\label{eq:theory_new:projected_target}
q_s(v)
\;=\;
\frac{p_0(v \mid s)^{1/T_{\textsf{train}}}\,\mathbf{1}\{v \in S_s\}}
     {\sum_{u \in S_s} p_0(u \mid s)^{1/T_{\textsf{train}}}}.
\end{equation}
By construction, $q_s$ is supported on $S_s$. Whenever $S_s \neq \mathcal{V}$ --- for example, if $k_{\textsf{train}}<|\mathcal{V}|$ or $\text{top-}p_{\textsf{train}}<1$ --- it lies on a proper face of the full simplex because it assigns zero probability outside the retained support.

At context $s$, \SD{} still uses ordinary cross-entropy:
\begin{equation}
\label{eq:theory_new:ce}
\mathcal{L}_s(\theta)
\;=\;
\mathrm{CE}\!\bigl(q_s,\,p_\theta(\cdot \mid s)\bigr)
\;=\;
\mathbb{E}_{v \sim q_s}\bigl[-\log p_\theta(v \mid s)\bigr].
\end{equation}
The important difference from naive self-training is therefore not the form of the loss. It is the fact that temperature and truncation alter the target before optimization begins. The rest of the theory is an analysis of what this target shift does.

\subsection{Understanding the \SD{} Objective and Its Learning Signal}
\label{app:theory_new:signal}
\label{app:theory_new:distribution} %

The first theoretical question is why \SD{} produces any learning signal at all. If one literally samples from a model and then trains the same model on those samples without changing the target, self-training is an on-policy fixed point. \SD{} avoids this fixed-point failure because temperature and truncation modify the teacher-induced target before the student sees it. We isolate these two sources of signal separately and then combine them.

\paragraph{Naive self-training is a fixed point.}
Consider first the degenerate case where the teacher samples from its own base distribution with unit temperature and no truncation. Then the training target is just the model itself, so the expected score-function gradient at initialization vanishes:
\begin{equation}
\label{eq:theory_new:score_function}
\mathbb{E}_{v \sim p_\theta(\cdot \mid s)}
\bigl[\nabla_\theta \log p_\theta(v \mid s)\bigr]
\;=\;
\nabla_\theta \sum_v p_\theta(v \mid s)
\;=\;
\nabla_\theta 1
\;=\;
0.
\end{equation}
The gradient of the log-probability, averaged under the model's own distribution, telescopes because probabilities sum to one. In practical terms, if one samples from the model's own distribution at $T{=}1$ and trains on those samples, the expected update direction is the zero vector. Naive self-training therefore produces no signal. Any useful signal in \SD{} must come from the way temperature and truncation modify the target before optimization begins.

\paragraph{Truncation introduces a support gate.}
To isolate the effect of truncation, first factor the loss through the retained support. Define the student's mass on the teacher's retained support by
\begin{equation}
\label{eq:theory_new:keptmass}
\textsf{KeptMass}_\theta(s)
\;\equiv\;
\sum_{v \in S_s} p_\theta(v \mid s),
\end{equation}
and the corresponding conditional distribution by
\begin{equation}
\label{eq:theory_new:restricted}
p_\theta(v \mid s, S_s)
\;=\;
\frac{p_\theta(v \mid s)\,\mathbf{1}\{v \in S_s\}}
     {\textsf{KeptMass}_\theta(s)}.
\end{equation}
Since $q_s$ is supported on $S_s$, the per-context loss can be written exactly as
\begin{equation}
\label{eq:theory_new:gate_cond}
\mathcal{L}_s(\theta)
\;=\;
-\log \textsf{KeptMass}_\theta(s)
\;+\;
\mathrm{CE}\!\bigl(q_s,\,p_\theta(\cdot \mid s, S_s)\bigr).
\end{equation}
This identity reveals that truncation splits the learning problem into two levels: a \emph{gate-level} objective that maximizes mass inside the retained support, and a \emph{conditional-level} objective that matches the teacher's within-support distribution. The gate cares only about the in/out partition; the conditional term cares only about relative probabilities inside the retained support. When $S_s=\mathcal{V}$, the gate term disappears and the factorization collapses to ordinary cross-entropy against the tempered teacher.

The same factorization also explains why tail suppression is persistent throughout training. When $S_s \neq \mathcal{V}$, the target $q_s$ lies on a proper face of the probability simplex $\Delta^{V-1}$, assigning exactly zero probability to every token outside $S_s$. Viewing the student at this context as an unconstrained softmax over local logits $z \in \mathbb{R}^{|\mathcal{V}|}$, we therefore have
\begin{equation}
\label{eq:theory_new:infimum}
\inf_{z \in \mathbb{R}^{|\mathcal{V}|}} \mathcal{L}_s(z)
\;=\;
H(q_s),
\end{equation}
but this infimum is not attained at any finite logit vector; it is approached only as outside-support logits tend to $-\infty$. Geometrically, the truncated target places the optimum on a simplex face, and the gate term drives the student toward that face. Training therefore never fully satisfies the gate penalty, maintaining persistent pressure to suppress tail logits throughout optimization.

\paragraph{Temperature reshapes the full support.}
To isolate the effect of temperature, now remove truncation and study the full-support target. If $S_s=\mathcal{V}$, the teacher target is $\textsf{Temper}_T[p_0(\cdot \mid s)]$, and the loss can be written as a R\'enyi-shaping term plus a KL anchor:
\begin{align}
\mathcal{L}_s(\theta)
&=
\mathbb{E}_{v \sim \textsf{Temper}_T[p_0(\cdot \mid s)]}
\bigl[-\log p_\theta(v \mid s)\bigr]
\nonumber\\
&=
(1-T)\,H_{1/T}\!\bigl(p_\theta(\cdot \mid s)\bigr)
+
T \cdot
\mathrm{KL}\!\Bigl(
\textsf{Temper}_T[p_0(\cdot \mid s)]
\;\|\;
\textsf{Temper}_T[p_\theta(\cdot \mid s)]
\Bigr)
+
T \cdot H\!\bigl(\textsf{Temper}_T[p_0(\cdot \mid s)]\bigr),
\end{align}
where
\[
H_\alpha(\pi)
\;=\;
\frac{1}{1-\alpha}\log\sum_v \pi(v)^\alpha
\]
is the R\'enyi entropy of order $\alpha \neq 1$. The R\'enyi entropy $H_{1/T}$ at order $\alpha=1/T$ interpolates between familiar extremes: as $T \to \infty$ ($\alpha \to 0$), it approaches $\log|\mathrm{supp}(\pi)|$, the maximum entropy achievable on the support; as $T \to 0$ ($\alpha \to \infty$), it approaches $-\log \max_v \pi(v)$, the min-entropy. Shannon entropy corresponds to the intermediate case $\alpha=1$ ($T=1$). For the typical \SD{} setting $T>1$, the order $1/T<1$ falls in the sub-Shannon regime, which is more sensitive to diffuse tails and less tolerant of concentrated peaks than Shannon entropy is. Throughout, occurrences of $(1-T)H_{1/T}$ are understood via the equivalent free-energy form $-T \log \sum_v \pi(v)^{1/T}$, with continuous extension at $T=1$, where the term equals zero.

The coefficient $(1{-}T)$ determines the direction of the resulting pressure. For $T>1$, $(1{-}T)<0$, so minimizing the loss \emph{maximizes} $H_{1/T}$ and therefore smooths the distribution. For $T<1$, the effect reverses and the loss sharpens the distribution. At $T=1$, the R\'enyi-shaping term vanishes identically and the fixed-point symmetry of \Cref{eq:theory_new:score_function} returns. The KL term keeps the student aligned with the teacher's tempered preferences.

This two-term structure already shows why even the pathological no-truncation, high-temperature regime of \Cref{sec:high_temp} retains a nontrivial learning signal: the R\'enyi-shaping term is nonzero whenever $T \neq 1$. But because the reshaping acts on the \emph{full vocabulary}, it lifts harmful tail tokens just as readily as it diversifies genuinely ambiguous contexts. Temperature alone creates exploration but does not know where that exploration should stop. Truncation supplies that boundary.

\paragraph{Full \SD{} combines both effects.}
Applying the same temperature decomposition inside the retained support and then reinserting the gate term yields the central three-term decomposition:
\begin{align}
\mathcal{L}_s(\theta)
&=
-\log \textsf{KeptMass}_\theta(s)
+
(1-T)\,H_{1/T}\!\bigl(p_\theta(\cdot \mid s, S_s)\bigr)
\nonumber\\
&\qquad
+
T \cdot
\mathrm{KL}\!\Bigl(
q_s
\;\|\;
\textsf{Temper}_T[p_\theta(\cdot \mid s, S_s)]
\Bigr)
+
T \cdot H(q_s).
\label{eq:three_term}
\end{align}

\emph{Proof sketch.}
Start from the gate-conditional factorization \Cref{eq:theory_new:gate_cond}. For the conditional cross-entropy term, write
\[
-\log p_\theta(v \mid s, S_s)
\;=\;
-T \log \textsf{Temper}_T^{S_s}[p_\theta(\cdot \mid s, S_s)](v)
\;-\;
T \log Z^{S_s}_{\theta,T},
\]
where $\textsf{Temper}_T^{S_s}[p_\theta(\cdot \mid s, S_s)]$ is the student's distribution restricted to $S_s$ and then tempered, and
$Z^{S_s}_{\theta,T} = \sum_{u \in S_s} p_\theta(u \mid s, S_s)^{1/T}$ is the corresponding within-support partition function. Taking the expectation under $q_s$ separates the first piece into
\[
T \cdot \mathrm{CE}\!\bigl(q_s,\,\textsf{Temper}_T^{S_s}[p_\theta(\cdot \mid s, S_s)]\bigr)
=
T \cdot H(q_s)
+
T \cdot \mathrm{KL}\!\bigl(q_s \,\|\, \textsf{Temper}_T^{S_s}[p_\theta(\cdot \mid s, S_s)]\bigr),
\]
yielding the KL anchor term and the constant $T \cdot H(q_s)$. The partition-function term becomes the R\'enyi-shaping term via the free-energy identity
\begin{equation}
\label{eq:theory_new:renyi_FE}
-T \log \sum_{u \in S} \pi(u)^{1/T}
\;=\;
(1-T)\,H_{1/T}(\pi),
\end{equation}
which holds for any distribution $\pi$ on a set $S$.

This decomposition is the objective-level core of the mechanism story in the paper. The first term compresses support, the second reshapes the retained head, and the third keeps that reshaping aligned with the teacher's relative preferences. The final term $T \cdot H(q_s)$ is constant in $\theta$ and does not contribute to optimization.

\paragraph{Immediate interpretation.}
The three-term decomposition makes clear what \SD{} is and is not doing. It is not learning from correctness labels, reward signals, or verification outcomes. Instead, it is fitting a target that has already been altered by the frozen model's own decoding rule. Truncation decides which part of the distribution is worth keeping; temperature decides how the retained mass is redistributed within that support; and the KL anchor prevents the student from drifting arbitrarily far from the frozen model's induced target.

\paragraph{Population limit.}
The same decomposition also clarifies what training is trying to approach at a fixed context. At the level of an unconstrained local softmax over logits, as the loss approaches its infimum, the student drives all of its mass onto the retained support and matches the truncated tempered teacher within that support. For truncated targets, this limit is reached only as outside-support logits tend to $-\infty$; at any finite logit vector, some residual outside-support mass remains. The student therefore does not converge toward the raw teacher distribution $p_0(\cdot \mid s)$; rather, it approaches the teacher \emph{after} training-time temperature and truncation have already reshaped that distribution. This is the first precise sense in which \SD{} can improve over the base model: the target is not the original distribution itself, but a structured transformation of it.

\paragraph{Logit-level gradient.}
The same mechanism becomes especially transparent at the logit level. At context $s$, the loss is $\mathrm{CE}(q_s, p_\theta(\cdot \mid s))$, so the standard softmax identity gives
\begin{equation}
\label{eq:theory_new:grad_split}
\frac{\partial \mathcal{L}_s}{\partial z_\theta(v \mid s)}
\;=\;
\begin{cases}
-(1-\textsf{KeptMass}_\theta(s))\,p_\theta(v \mid s, S_s)
\;+\;
\bigl(p_\theta(v \mid s, S_s)-q_s(v)\bigr),
& v \in S_s, \\[4pt]
p_\theta(v \mid s),
& v \notin S_s.
\end{cases}
\end{equation}
For tokens inside $S_s$, the gradient splits into two additive components: a support-transfer term $-(1-\textsf{KeptMass}_\theta)\,p_\theta(v \mid s, S_s)$ that pulls mass from outside into the retained support, and a within-support fitting term $p_\theta(v \mid s, S_s) - q_s(v)$ that reshapes the head toward the teacher target. For tokens outside $S_s$, the target is zero, so the gradient reduces to $+p_\theta(v \mid s)$: strictly positive, meaning gradient descent pushes those logits downward directly. Tail suppression is therefore not an indirect side effect of fitting the head; it is written into the update rule itself as an explicit downward force on every outside-support token.

This gradient structure also clarifies the relationship between \SD{} and neighboring paradigms. Policy-gradient reinforcement learning breaks the self-training fixed point by weighting the score function with an external return. \SD{} breaks it by altering the target distribution itself, so optimization remains standard supervised learning with positive, normalized weights. Standard knowledge distillation, by contrast, matches a full-vocabulary teacher and therefore lacks the support-compression term entirely; without the gate term, there is no mechanism to drive aggressive tail suppression.

\paragraph{Relation to Shannon entropy minimization.}
This objective is also distinct from direct Shannon entropy minimization. As the preceding temperature-only decomposition already shows, \SD{} induces a R\'enyi-shaping term of order $1/T$ rather than a Shannon entropy term, and the shaping order varies continuously with the training temperature. Operationally, \SD{} still remains ordinary supervised fine-tuning on teacher-induced targets with positive, normalized weights, rather than an objective that directly optimizes Shannon entropy using signed policy-gradient-like weights.

\paragraph{Summary.}
The objective-level picture is now complete. Naive self-training fails because it is a fixed point. Truncation creates a support gate that pushes probability mass onto the retained support. Non-unit temperature reshapes the target inside that support. Full \SD{} combines these two forces, and the resulting gradients make explicit why the student can learn to suppress tail mass even without any correctness labels. The next subsection explains why this same global objective suppresses diffuse lock tails while preserving useful exploration at fork-like contexts.

\subsection{How \SD{} Reshapes Locks and Forks}
\label{app:theory_new:locks_forks}
\label{app:theory_new:eval} %
\label{app:temp_composition} %

The same global \SD{} objective does not act the same way at every context because the local retained support is not the same everywhere. Lock-like contexts are those where useful mass is concentrated on one or a few top tokens; fork-like contexts are those where several plausible continuations survive. This difference in local support geometry is enough to explain both the training-time asymmetry and the evaluation-time asymmetry emphasized in the main text.

\paragraph{Truncation makes the objective context-adaptive.}
Without truncation, the R\'enyi-shaping term in \Cref{eq:three_term} acts on the full vocabulary: for $T>1$ it smooths the distribution globally, lifting useful fork alternatives and harmful lock distractors alike. Truncation changes this by restricting the reshaping to the retained support $S_s$. The set $S_s$ is selected by the teacher's local distributional shape. A peaked distribution reaches the top-$p$ threshold in very few tokens, while a flatter distribution retains many. The same global training rule therefore produces different local behavior depending on the size and geometry of $S_s$.

\paragraph{At locks, support compression dominates.}
When $S_s$ is small (a single dominant token plus perhaps one runner-up), the R\'enyi-shaping term
$(1{-}T)\,H_{1/T}(p_\theta(\cdot \mid s, S_s))$ has limited room to act, because $H_{1/T}(p_\theta(\cdot \mid s, S_s)) \le \log |S_s|$ and the head entropy of a near-singleton distribution is already close to zero. The learning signal at a lock therefore comes primarily from the support-compression term $-\log \textsf{KeptMass}_\theta(s)$, which pushes probability onto the retained support and directly suppresses the distractor tail. At the logit level, every token outside $S_s$ receives a gradient of $+p_\theta(v \mid s)$ by \Cref{eq:theory_new:grad_split}, driving its logit downward in proportion to its current mass. Because the truncated target lies on a proper face of the simplex, this downward pressure never fully disappears at any finite logit vector. The net effect is that lock-like contexts lose diffuse tail mass and typically become easier to secure under evaluation-time decoding; in the extreme case $|S_s|=1$, this reduces to pure support compression.

\paragraph{At forks, within-support reshaping has room to act.}
When $S_s$ is larger (several plausible continuations survive truncation), the support-compression term is still active, but the retained head already contains most of the useful mass. In this regime, the R\'enyi-shaping term has room to matter. For $T>1$, the coefficient $(1{-}T)<0$ means that minimizing the loss maximizes $H_{1/T}(p_\theta(\cdot \mid s, S_s))$, smoothing the distribution among the surviving alternatives. Crucially, this smoothing is confined to the retained support and cannot reopen the discarded tail. The KL anchor keeps this reshaping aligned with the teacher's within-support preferences, so the head is flattened only within the set of tokens that the frozen model has already judged worth keeping. This is the formal reason the same objective can preserve useful diversity at fork-like contexts while still cleaning up the tail elsewhere.

\paragraph{Temperature sensitivity within fixed support.}
The preceding discussion explains the training-time asymmetry. We now ask how evaluation-time temperature interacts with these reshaped distributions. Let $\tau=T_{\textsf{eval}}$. For algebraic convenience, write $\gamma=1/\tau$ and study temperature inside a fixed retained support.

The relevant local object is the teacher restricted to $S_s$ and retempered by the power $\gamma$:
\begin{equation}
\label{eq:theory_new:barp}
\pi_{s,\gamma}(v)
\;=\;
\frac{\mathbf{1}\{v \in S_s\}\,p_0(v \mid s)^\gamma}
     {\sum_{u \in S_s} p_0(u \mid s)^\gamma}.
\end{equation}
Differentiating $\pi_{s,\gamma}$ shows that every temperature-sensitivity question reduces to a covariance identity:
\begin{equation}
\label{eq:theory_new:escort_cov}
\frac{d}{d\gamma}\,
\mathbb{E}_{v \sim \pi_{s,\gamma}}[f(v)]
\;=\;
\mathrm{Cov}_{\pi_{s,\gamma}}\!\bigl(f(v),\,\log p_0(v \mid s)\bigr).
\end{equation}
We now apply this identity in several ways.

\paragraph{Useful sets and direction of reshaping.}
Let $A \subseteq S_s$ be a nonempty set of locally useful actions such that $\pi_{s,\gamma}(A)>0$. At a lock, $A$ may be a single correct continuation; at a fork, it may be a set of viable branches. Applying \Cref{eq:theory_new:escort_cov} with
$f(v)=\mathbf{1}\{v \in A\}$ gives the absolute sensitivity
\begin{equation}
\label{eq:theory_new:good_action}
\frac{d}{d\gamma}\pi_{s,\gamma}(A)
\;=\;
\mathrm{Cov}_{\pi_{s,\gamma}}\!\bigl(
\mathbf{1}\{v \in A\},\,\log p_0(v \mid s)
\bigr).
\end{equation}
Dividing both sides by $\pi_{s,\gamma}(A)$ and expanding the covariance gives the proportional sensitivity
\begin{equation}
\label{eq:theory_new:direction_reshaping}
\frac{\partial}{\partial \gamma}
\log \pi_{s,\gamma}(A)
\;=\;
\mathbb{E}_{\pi_{s,\gamma}(\cdot \mid A)}[\log p_0(v \mid s)]
\;-\;
\mathbb{E}_{\pi_{s,\gamma}}[\log p_0(v \mid s)].
\end{equation}
The right-hand side compares the average log-probability of tokens in $A$ to the within-support average. When these two averages differ, tempering redistributes mass between $A$ and its complement.

This criterion captures the canonical lock and fork regimes. At a lock, the correct token typically has log-probability above the within-support average, so increasing $\gamma$ (lowering temperature) concentrates mass further on that token. At a fork, viable but lower-ranked branches can lie below the within-support average, so decreasing $\gamma$ (raising temperature) redistributes mass toward them. Because the tail has already been removed by truncation, this redistribution acts within a cleaned-up head rather than reviving the discarded tail.

\paragraph{Entropy sensitivity.}
The same covariance identity also controls how entropy responds to temperature inside a fixed retained support, and this is the result that we will use directly in the next subsection. Differentiating the entropy of $\pi_{s,\gamma}$ gives
\begin{equation}
\label{eq:theory_new:entropy_response}
\frac{d}{d\gamma}H(\pi_{s,\gamma})
\;=\;
-\gamma\,\mathrm{Var}_{\pi_{s,\gamma}}\!\bigl(\log p_0(v \mid s)\bigr)
\;\le\;
0.
\end{equation}
Equivalently, since $\gamma=1/T$ and $d\gamma/dT=-1/T^2<0$, the chain rule gives
$dH/dT = \mathrm{Var}/T^3 \ge 0$: entropy is nondecreasing in temperature. The variance term therefore summarizes local temperature sensitivity: it vanishes for singleton supports and is also small for nearly uniform heads, while it becomes large when several viable alternatives remain with non-identical probabilities. After \SD{}, lock-like contexts typically retain either a tiny support or an almost degenerate head, whereas fork-like contexts can retain a nontrivial, uneven multi-token head. Evaluation-time temperature is therefore typically more effective at forks and less effective at locks, which is the asymmetry needed to alleviate the precision-exploration conflict.

\paragraph{Evaluation-time behavior under a local ideal-fit approximation.}
The preceding analysis characterizes how the objective behaves locally. To connect that picture to evaluation time, we now use a local ideal-fit approximation and then state the two main consequences that follow from it.

At a teacher-visited context $s$, suppose the student has fit the training target:
\begin{equation}
\label{eq:theory_new:ideal_fit}
p_\theta(\cdot \mid s)
\;=\;
q_s.
\end{equation}
This should be read as a local approximation rather than as a claim that every trained student exactly satisfies it at every context. We also write
\[
p_{0,\tau}(\cdot \mid s)
\;\equiv\;
\textsf{Temper}_{\tau}[p_0(\cdot \mid s)]
\]
for the frozen model after evaluation-time temperature and before any new support restriction. Under this approximation, the formulas below cleanly separate the contributions of training-time temperature and training-time truncation.

\paragraph{Temperature composition inside fixed support.}
Inside a fixed retained support, training-time and evaluation-time temperatures compose multiplicatively.

\begin{lemma}[Temperatures compose multiplicatively]
\label{lem:temp_compose}
For any distribution $p$ over $\mathcal{V}$ and temperatures $T_1, T_2 > 0$,
\begin{equation}
\label{eq:temp_compose}
\textsf{Temper}_{T_2}\!\left[\textsf{Temper}_{T_1}[p]\right]
=
\textsf{Temper}_{T_1 \cdot T_2}[p].
\end{equation}
The same holds when $p$ is replaced by any restriction to a fixed support set
$S \subseteq \mathcal{V}$.
\end{lemma}

\begin{proof}
$\textsf{Temper}_{T_2}[\textsf{Temper}_{T_1}[p]](v) \propto
\left(p(v)^{1/T_1}\right)^{1/T_2} = p(v)^{1/(T_1 T_2)}$, and renormalization constants cancel.
\end{proof}

\begin{proposition}[Evaluation-time form under local ideal fit]
Assume \Cref{eq:theory_new:ideal_fit} at context $s$, and let
$\tau=T_{\textsf{eval}}$ denote the evaluation-time temperature. Applying temperature to the student gives
\begin{equation}
\label{eq:theory_new:q_tau}
q_{s,\tau}(v)
\;=\;
\frac{q_s(v)^{1/\tau}}{\sum_u q_s(u)^{1/\tau}}
\;=\;
\frac{\mathbf{1}\{v \in S_s\}\,p_0(v \mid s)^{1/(T_{\textsf{train}}\tau)}}
     {\sum_{u \in S_s} p_0(u \mid s)^{1/(T_{\textsf{train}}\tau)}}.
\end{equation}
Inside a fixed retained support, the student therefore behaves like the teacher evaluated at the product temperature
$T_{\textsf{eff}} = T_{\textsf{train}}T_{\textsf{eval}}$. Under the local ideal-fit approximation, this is the cleanest local formal version of the effective-temperature picture in the experiments.
\end{proposition}

\begin{proposition}[Local gain decomposition under local ideal fit]
Under the same local approximation, the student's evaluation-time gain separates into a support-compression factor and a within-support reshaping factor. Reuse the restricted retempered teacher distribution $\pi_{s,\gamma}$ from \Cref{eq:theory_new:barp}. We also need one scalar quantity: the teacher's evaluation-time mass that remains inside the training-time retained support,
\begin{equation}
\label{eq:theory_new:m_tau}
m_s(\tau)
\;\equiv\;
p_{0,\tau}(S_s)
\;=\;
\sum_{v \in S_s} p_{0,\tau}(v \mid s).
\end{equation}
Then for any event $A \subseteq S_s$ with $\pi_{s,\,1/\tau}(A)>0$,
\begin{equation}
\label{eq:theory_new:local_gain}
q_{s,\tau}(A)
\;=\;
\frac{1}{m_s(\tau)}
\cdot
\frac{\pi_{s,\,1/(T_{\textsf{train}}\tau)}(A)}
     {\pi_{s,\,1/\tau}(A)}
\cdot
p_{0,\tau}(A \mid s).
\end{equation}
This identity separates two improvement channels. The factor $1/m_s(\tau)$ is a support-compression gain: mass that the teacher would have leaked outside the retained support is recovered. The ratio of the two within-support conditionals is a reshaping gain: the student redistributes probability among the retained tokens themselves. The decomposition is algebraically exact under the ideal-fit assumption and is useful because each factor has a clean limiting interpretation.
\end{proposition}

\begin{remark}[Two limiting cases]
Two limiting cases isolate the roles of the two ingredients in \SD{}. If training-time truncation is absent so that $S_s=\mathcal{V}$, then $m_s(\tau)=1$ and
\begin{equation}
\label{eq:theory_new:no_trunc}
q_{s,\tau}(v)
\;=\;
p_{0,T_{\textsf{train}}\tau}(v \mid s).
\end{equation}
In this regime, \SD{} reduces to pure temperature composition.

If instead $T_{\textsf{train}}=1$, then the reshaping ratio in
\Cref{eq:theory_new:local_gain} becomes $1$, and for any event
$A \subseteq S_s$,
\begin{equation}
\label{eq:theory_new:unit_temp}
q_{s,\tau}(A)
\;=\;
\frac{p_{0,\tau}(A \mid s)}{p_{0,\tau}(S_s \mid s)}.
\end{equation}
In this regime, \SD{} reduces to pure support compression inside the retained support. Without truncation there is no support-compression gain; without non-unit training temperature there is no within-support reshaping gain.
\end{remark}

\paragraph{Summary.}
The lock-fork asymmetry does not require different objectives or context-specific hyperparameters. It arises because the same global objective acts on different retained-support geometries. At locks, support compression dominates, removing diffuse tail mass and making the remaining head more robust to evaluation-time decoding. At forks, within-support reshaping has room to preserve and redistribute useful alternatives. Under a local ideal-fit approximation, the same picture carries through to evaluation time: training-time and evaluation-time temperatures compose inside the retained support, and the student's local gain separates into a support-compression channel and a within-support reshaping channel. The next subsection focuses on one especially important consequence of this picture: the student can become lower-entropy overall while preserving the conditional head entropy needed for exploration.

\subsection{Why \SD{} Can Lower Total Entropy While Preserving Conditional Head Entropy for Exploration}
\label{app:theory_new:entropy}

A central empirical pattern in the paper is that the \SD{} student can become more concentrated overall while remaining more exploitable by evaluation-time temperature. At first glance, these two statements can seem to pull in opposite directions: if the model is lower-entropy after training, why does it still preserve the kind of diversity that supports exploration? The key point is that these statements concern different objects. Total entropy measures uncertainty over the full vocabulary, whereas evaluation-time temperature acts on the conditional distribution inside the retained support. We now make that distinction explicit.

\paragraph{Exact entropy decomposition.}
To separate these effects, write the student's conditional distribution on the complement of $S_s$ as
\[
u_\theta(v \mid s)
\;=\;
\frac{p_\theta(v \mid s)\,\mathbf{1}\{v \notin S_s\}}
     {1-\textsf{KeptMass}_\theta(s)},
\]
whenever $1-\textsf{KeptMass}_\theta(s) > 0$. Expanding the Shannon entropy of
$p_\theta(\cdot \mid s)$ over the disjoint sets $S_s$ and $S_s^c$ gives
\begin{equation}
\label{eq:theory_new:entropy_decomp}
H\!\bigl(p_\theta(\cdot \mid s)\bigr)
\;=\;
\underbrace{h_2\!\bigl(\textsf{KeptMass}_\theta(s)\bigr)}_{\textup{gate entropy}}
\;+\;
\underbrace{\textsf{KeptMass}_\theta(s)\,
H\!\bigl(p_\theta(\cdot \mid s, S_s)\bigr)}_{\textup{head entropy}}
\;+\;
\underbrace{\bigl(1-\textsf{KeptMass}_\theta(s)\bigr)\,
H\!\bigl(u_\theta(\cdot \mid s)\bigr)}_{\textup{tail entropy}},
\end{equation}
where $h_2(\pi) = -\pi \log \pi - (1-\pi)\log(1-\pi)$ is the binary entropy function.

This decomposition separates three distinct channels through which \SD{} can change total entropy. First, as the student moves more probability mass onto the retained support, the binary gate term changes. Second, as outside-support mass shrinks, the contribution of the tail shrinks with it. Third, the conditional entropy of the retained head can itself change, either decreasing at lock-like contexts or increasing at fork-like contexts depending on how much room the retained support leaves for within-head reshaping.

\paragraph{Why total entropy can still fall when $T_{\textsf{train}}>1$.}
With the three channels in hand, the apparent paradox becomes straightforward to resolve. The support-compression mechanism identified earlier reduces total entropy through two large-scale effects at once: it suppresses diffuse outside-support mass, and in the high-retained-mass regime relevant here it also lowers the binary gate term by pushing more probability mass onto the retained support.

The within-support reshaping term can act differently. At lock-like contexts, the retained head is already close to singleton, so within-head entropy has little room to increase and may decrease further. At fork-like contexts, by contrast, $T_{\textsf{train}}>1$ can flatten the retained head and therefore increase its conditional entropy locally. But this local increase is bounded by the size of the retained support, whereas the gate and tail reductions operate on the entire complement of that support. Total entropy can therefore decrease even when the student preserves or even increases conditional head entropy at the subset of contexts where exploration remains useful.

\paragraph{Evaluation-time temperature acts on conditional head entropy.}
The operational role of evaluation-time temperature is now easy to state. Temperature does not act on the full-vocabulary entropy decomposition directly; it acts on the conditional distribution inside whatever head remains after training and truncation. Applying the fixed-support entropy-response calculation to the retained-head distribution gives
\begin{equation}
\label{eq:theory_new:eval_entropy_response}
\frac{d}{d\tau}
H\!\Bigl(\textsf{Temper}_{\tau}^{S_s}[p_\theta(\cdot \mid s, S_s)]\Bigr)
\;=\;
\frac{
\mathrm{Var}_{\textsf{Temper}_{\tau}^{S_s}[p_\theta(\cdot \mid s, S_s)]}
\!\bigl[\log p_\theta(v \mid s, S_s)\bigr]
}{\tau^3}
\;\ge\;
0.
\end{equation}
If the retained head is effectively singleton, or nearly uniform, the variance term is near zero and temperature has little effect. If the retained head contains several comparable but non-identical tokens, the variance term is substantial and evaluation-time temperature changes the operational policy much more strongly. This derivative therefore gives the formal version of the main-text intuition: lower total entropy and stronger exploration are not contradictory because they concern different levels of the distribution.

\paragraph{Why the teacher can be nearly temperature-inert while the student need not be.}
The variance criterion in \Cref{eq:theory_new:eval_entropy_response} also helps explain the empirical asymmetry between frozen model and student observed in \Cref{fig:empirical_triptych}. At many contexts, the frozen model's evaluation-time distribution is already dominated by a single token: the kept set under top-$k$/top-$p$ is effectively singleton, and the log-probability variance within that singleton is therefore negligible. In such a context, changing $\tau$ barely changes the operational policy because there is no meaningful within-head spread for temperature to act on.

\SD{} changes this asymmetrically. At lock-like contexts, the retained head can stay tiny, so decoding remains nearly temperature-inert. At fork-like contexts, training can suppress outside-support mass while preserving a nontrivial multi-token head, so evaluation-time temperature remains an effective control knob. In this sense, \SD{} removes the wrong kind of uncertainty while preserving the right kind.

\paragraph{Connection back to the objective.}
The entropy picture is not a separate mechanism; it is another view of the same three-term objective from \Cref{eq:three_term}. The gate term $-\log \textsf{KeptMass}_\theta(s)$ drives support acquisition and tail removal, typically lowering the gate term in the high-retained-mass regime and shrinking the tail contribution in \Cref{eq:theory_new:entropy_decomp}. The R\'enyi-shaping term acts only inside the retained head, and for the typical \SD{} regime $T>1$, it smooths that head rather than the full vocabulary. The KL anchor keeps this smoothing aligned with the teacher's retained preferences.

This is why the lock-fork distinction reappears naturally inside the entropy decomposition. At locks, the retained head is nearly singleton, so the visible effect is tail suppression. At forks, the retained head contains several plausible continuations, so the same objective can remove tail mass while preserving or increasing uncertainty inside the head itself. The student can therefore be globally sharper yet locally more explorable.

\paragraph{Summary.}
The apparent contradiction between lower total entropy and preserved exploration is only a mismatch of levels. \SD{} can lower full-vocabulary entropy by moving probability mass onto the retained support and suppressing diffuse outside-support mass. At the same time, it can preserve or increase conditional head entropy at the contexts where several plausible continuations survive truncation. Evaluation-time temperature acts on that conditional head, not on the discarded tail. The next subsection uses this same perspective to explain why decode-only tuning on the frozen model cannot recover the effect of changing the model itself.

\subsection{Why Decode-Only Tuning Cannot Match \SD{}}
\label{app:theory_new:decode_only}

The previous subsections characterize what \SD{} changes during training. We now ask whether a \emph{single global} decode-only policy on the frozen model could reproduce the same effect. In the special no-truncation ideal-fit case from \Cref{eq:theory_new:no_trunc}, the answer can be yes: temperature composition alone makes local matching possible. The relevant question for the truncated regime studied in the paper is different: can one global decode-only policy match the student across heterogeneous contexts? In general, no. Under the local ideal-fit approximation, the student at a fixed context still has the form of a power transform on a retained prefix; the limitation is that decode-only tuning must apply one global policy to the frozen model's original cumulative geometry, whereas \SD{} changes that geometry itself.

\paragraph{Decoding operators.}
Let $p_{(1)}(s) \ge p_{(2)}(s) \ge \cdots \ge p_{(V)}(s)$ be the frozen-model probabilities at context $s$, sorted by decreasing probability. Write $\alpha = 1/\tau$, where $\tau=T_{\textsf{eval}}$. A decode-only policy composes three operators in some fixed order $\sigma$.

\emph{Temperature scaling:}
\begin{equation}
\label{eq:theory_new:op_temp}
(\mathcal{T}_\alpha\, p)_{(i)}
\;=\;
\frac{p_{(i)}^\alpha}{\sum_{j=1}^{V} p_{(j)}^\alpha}.
\end{equation}

\emph{Top-$k$ truncation:}
\begin{equation}
\label{eq:theory_new:op_topk}
(\mathcal{K}_k\, p)_{(i)}
\;=\;
\frac{p_{(i)}\,\mathbf{1}[i \le k]}{\sum_{j=1}^{k} p_{(j)}}.
\end{equation}

\emph{Top-$p$ truncation}, with prefix length
$m_{\text{top-}p}(p)=\min\{m : \sum_{i=1}^{m} p_{(i)} \ge \text{top-}p\}$:
\begin{equation}
\label{eq:theory_new:op_topp}
(\mathcal{P}_{\text{top-}p}\, p)_{(i)}
\;=\;
\frac{p_{(i)}\,\mathbf{1}[i \le m_{\text{top-}p}(p)]}{\sum_{j=1}^{m_{\text{top-}p}(p)} p_{(j)}}.
\end{equation}
These operators slightly extend the empirical sweep in the paper: the experiments focus on the standard vLLM ordering, while the analysis below asks what any fixed ordering of the same ingredients can and cannot do.

\paragraph{Normal form: all operator orderings collapse.}
A natural objection is that the empirical sweep in the paper tests only one practical operator ordering, namely temperature $\to$ top-$k$ $\to$ top-$p$ as implemented by vLLM. Perhaps a different ordering would close the gap. The following proposition shows that it cannot: all fixed orderings collapse to the same restricted normal form.

\begin{proposition}[Normal form of decode-only policies]
\label{prop:theory_new:normal_form}
For any fixed permutation $\sigma$ of $\mathcal{T}_\alpha$, $\mathcal{K}_k$, and $\mathcal{P}_{\text{top-}p}$, the final decode-only distribution can be written as
\begin{equation}
\label{eq:theory_new:normal_form}
\mu^\sigma_{s}((i))
\;=\;
\frac{p_{(i)}(s)^\alpha\,\mathbf{1}[i \le m^\sigma_s]}
{\sum_{j=1}^{m^\sigma_s} p_{(j)}(s)^\alpha},
\end{equation}
for some prefix length $m^\sigma_s$ that depends on the order $\sigma$, the parameters $(\alpha,k,\text{top-}p)$, and the context $s$.
\end{proposition}

\begin{proof}
Each operator preserves rank-prefix structure.
First, $\mathcal{T}_\alpha$ is monotone in $p$, so it preserves the ranking and keeps the full support.
Second, $\mathcal{K}_k$ keeps the top-$k$ prefix.
Third, $\mathcal{P}_{\text{top-}p}$ keeps the smallest prefix whose cumulative mass reaches the chosen top-$p$ threshold.

Therefore, any sequence of these operators produces a distribution supported on some prefix of the frozen-model ranking. Once the support has been reduced to a prefix $\{(1),\dots,(m)\}$, applying temperature yields probabilities proportional to $p_{(i)}^\alpha$ on that same prefix. The only thing an ordering can change is where the prefix boundary lands, because top-$p$ is evaluated against different intermediate distributions depending on whether temperature has already been applied. It cannot change the fact that the final decoder acts as a single power transform on a prefix of the original ranking.
\end{proof}

This proposition already shows the basic limitation of decode-only tuning. Reordering can move the prefix boundary, but it cannot create a new kind of distributional transformation. It also clarifies why this does not contradict the local ideal-fit picture above: at a single context, the student can still lie in the same normal-form family, but training changes the underlying cumulative curves and ranking that one global decoder must serve across contexts.

\paragraph{Two structural invariants.}
The normal form immediately implies two constraints that no reordering can break.

\begin{corollary}[Prefix rigidity]
\label{cor:theory_new:prefix_rigidity}
For any order $\sigma$,
$\mathrm{supp}(\mu^\sigma_s)=\{(1),\dots,(m^\sigma_s)\}$.
To include rank-$(r)$ token, the decoder must also include every higher-ranked token $(1),\dots,(r{-}1)$.
\end{corollary}

\begin{corollary}[Power rigidity]
\label{cor:theory_new:power_rigidity}
For any surviving pair $i,j \le m^\sigma_s$,
\begin{equation}
\label{eq:theory_new:power_rigidity}
\log\frac{\mu^\sigma_s((i))}{\mu^\sigma_s((j))}
\;=\;
\alpha\,\log\frac{p_{(i)}(s)}{p_{(j)}(s)}.
\end{equation}
All pairwise log-odds inside the kept support are scaled by the same global factor $\alpha = 1/T_{\textsf{eval}}$.
\end{corollary}

These two rigidities are the structural reason the decode-only sweep curves in \Cref{fig:ssd_vs_sweep} are so flat. Prefix rigidity means a lower-ranked useful branch cannot be admitted without also admitting every higher-ranked token above it, even if some of those higher-ranked tokens are distractors. Power rigidity means the decoder cannot flatten one part of the head while sharpening another, cannot widen the head-tail gap without simultaneously changing within-head ratios in the same global way, and cannot treat some head tokens as useful alternatives while suppressing others as noise.

\paragraph{The standard pipeline makes the coupling explicit.}
The normal-form result already shows that decode-only tuning is limited, but the standard practical pipeline makes the conflict especially transparent. Fix a context $s$. After temperature and top-$k$, the surviving distribution is
\begin{equation}
\label{eq:theory_new:decode_pi_tauk}
\tilde{\pi}^{(\tau,k)}_{(i)}(s)
\;=\;
\frac{p_{(i)}(s)^{1/\tau}\,\mathbf{1}[i \le k]}
{\sum_{j=1}^{k} p_{(j)}(s)^{1/\tau}}.
\end{equation}
Top-$p$ then keeps the smallest prefix whose cumulative mass reaches the chosen threshold. Define the prefix mass by
\begin{equation}
\label{eq:theory_new:decode_prefix_mass}
S_{s,m}(\tau,k)
\;=\;
\frac{\sum_{i=1}^{m} p_{(i)}(s)^{1/\tau}}
{\sum_{j=1}^{k} p_{(j)}(s)^{1/\tau}},
\qquad m \le k.
\end{equation}
Applying the escort-covariance identity from \Cref{eq:theory_new:escort_cov} to the top-$k$ restricted distribution gives
\begin{equation}
\label{eq:theory_new:Sm_monotone}
\frac{d}{d\tau} S_{s,m}(\tau,k) \;\le\; 0.
\end{equation}
So increasing temperature makes the top of the ranked list accumulate mass more slowly, forcing the decoder to retain more tokens to reach the same top-$p$ threshold. The resulting prefix length
\[
m_s(\tau,k,\text{top-}p)
=
\min\{m \le k : S_{s,m}(\tau,k) \ge \text{top-}p\}
\]
is therefore nondecreasing in $\tau$, $k$, and top-$p$.

This gives a very concrete interpretation of the three practical decoding knobs. Increasing $\tau$ makes the retained head \emph{flatter}, but it also makes the retained prefix \emph{longer}. Increasing $k$ or top-$p$ makes the retained prefix longer, but neither changes the internal geometry of the head in any context-dependent way. There is no decode-only knob that flattens a useful fork head while leaving the support boundary of lock-like contexts unchanged. The knob that helps forks is exactly the knob that destabilizes locks.

Under the standard pipeline, a single decode-only policy can satisfy both lock and fork requirements simultaneously only if $r_{\mathrm{F}} \le k$ and
\begin{equation}
\label{eq:theory_new:feasible_interval}
S_{s_{\mathrm{F}},\,r_{\mathrm{F}}-1}(\tau,k)
\;<\;
\text{top-}p
\;\le\;
S_{s_{\mathrm{L}},\,r_{\mathrm{L}}}(\tau,k).
\end{equation}
By \Cref{eq:theory_new:Sm_monotone}, increasing $\tau$ to help the fork simultaneously lowers the lock-side upper bound. The same knob that makes it easier to retain more fork alternatives therefore makes it harder to keep lock supports short. This is the precision-exploration conflict in its most operational form.

Reordering the operators can shift where the prefix boundary lands, and when top-$p$ precedes top-$k$ the final support can be clipped again by $k$. But no reordering escapes the two rigidities above: the final decoder still acts as a single global exponent on a prefix of the frozen ranking. Reordering can shift the compromise; it cannot create a context-dependent transformation of the frozen model's cumulative geometry.

\paragraph{Why \SD{} has a degree of freedom decode-only tuning lacks.}
\SD{} escapes this limitation because training changes the base distribution itself from $p_0(\cdot \mid s)$ to $p_\theta(\cdot \mid s)$. Once the distribution changes, the cumulative curves seen by the decoder can change as well:
\[
S_{s,m}(\tau,k; p_0)
\;\longrightarrow\;
S_{s,m}(\tau,k; p_\theta),
\]
This is the degree of freedom that no decode-only reordering possesses. In the truncated regime, \SD{} can remove diffuse tail mass at lock-like contexts while preserving a cleaner multi-token head at fork-like contexts. When those cumulative-curve changes move in opposite directions at locks and forks, the feasible interval in \Cref{eq:theory_new:feasible_interval} widens. This is the sense in which \SD{} changes the problem that the decoder is solving: the decoder is no longer operating on the same cumulative geometry.

Even beyond support boundaries, power rigidity from \Cref{eq:theory_new:power_rigidity} still constrains any reordered decoder to a single global exponent on all surviving log-odds. \SD{}, by contrast, changes the underlying logits directly: ranks can move, head-tail gaps can widen context-dependently, cumulative curves can change differently across contexts, and the student can become simultaneously more concentrated and more temperature-responsive. That is the structural change that persists in the empirical decode-only sweeps.

\paragraph{The ``truncate first'' objection.}
The most natural counterargument is to choose the support first with top-$p$ and only then use temperature for exploration inside that fixed support. This partially decouples support from temperature, but it does not solve the real problem. The support is still chosen from the frozen model's own cumulative curve. If that curve concentrates mass poorly at a lock, the mistake is frozen in before temperature acts. If it under-represents a useful fork branch, truncating first may exclude that branch entirely, and later temperature cannot bring it back.

So support-first decoding shifts the compromise, but it does not create the main \SD{} effect: locks becoming easier to secure and fork heads becoming cleaner without reopening the tail. The missing ingredient is still the same one: changing the model's distribution itself rather than only changing how that fixed distribution is decoded.

\paragraph{Summary.}
The distinction between \SD{} and decode-only tuning is structural, not merely parametric. All fixed orderings of temperature, top-$k$, and top-$p$ collapse to a power transform on a prefix of the frozen model's ranking. This induces prefix rigidity and power rigidity, and under the standard pipeline it ties exploration and precision to the same global decoding knob. By changing the model distribution itself, \SD{} can alter the cumulative curves and ranking that the global decoder sees. That is why a decode-only gap can persist even after the best sweep on the frozen model.

\section{Experimental Details and Additional Analyses}
\label{app:extended_experimental}

This section provides the experimental details and supplementary empirical analyses that support the results in \Cref{sec:experiments,sec:mechanism}. We first give the full training and evaluation protocol, then expand the hyperparameter and transfer results from Section~\ref{sec:experiments}, and finally provide the additional empirical details behind the toy simulation and the high-temperature stress test from Section~\ref{sec:mechanism}.

\subsection{Full Experimental Setup}
\label{app:experimental_details}

This subsection fully specifies the setup summarized in \Cref{sec:experiments}. The key point is that the training data consists only of competitive-programming prompts and the model's own sampled solutions; no verifier, execution filter, or external teacher is used at any stage.

\textbf{Prompt source and synthetic data generation.}
All \SD{} training data is synthesized from the seed subset of the rSTARcoder dataset~\citep{liu2025rstarcoder}, used only as a pool of unlabeled competitive-programming prompts. After exact string de-duplication on the whitespace-normalized problem statement, this yields ${\sim}$10{,}168 unique problems. For each prompt, we sample exactly one solution from the frozen base model using vLLM~\citep{kwon2023vllm} (v0.11.0, tensor-parallel across 8 GPUs) with a 128K maximum sequence length. The per-model decoding configuration used for this sampling step is listed in \Cref{tab:sd_settings}. We apply only a minimal degeneracy filter to remove clearly unusable outputs, such as empty responses and single-line stubs. No correctness verification of any kind is applied; the retained samples are the raw, unverified \SD{} training targets.

\begin{table}[t]
\centering
\caption{\textbf{Generation-time and evaluation-time decoding settings used throughout the paper.} All configurations use 128K maximum sequence length and $N{=}1$ sample per prompt during synthetic data generation.}
\label{tab:sd_settings}
\small
\setlength{\tabcolsep}{4pt}
\begin{tabular}{l ccc ccc}
\toprule
& \multicolumn{3}{c}{\textbf{Generation}} & \multicolumn{3}{c}{\textbf{Evaluation}} \\
\cmidrule(lr){2-4} \cmidrule(lr){5-7}
\textbf{Model} & $T_{\textsf{train}}$ & top-$k$ & top-$p$ & $T_{\textsf{eval}}$ & top-$k$ & top-$p$ \\
\midrule
Llama-3.1-8B-Instruct & 0.8 & 20 & 0.8 & 0.7 & 20 & 0.8  \\
Qwen3-4B-Instruct  & 1.6 & 20 & 0.8 & 1.1 & 20 & 0.8  \\
Qwen3-4B-Thinking  & 1.1  & 20 & 0.95 & 0.7 & 20 & 0.95 \\
Qwen3-30B-Instruct & 1.6 & 20 & 0.8 & 0.9 & 20 & 0.8  \\
Qwen3-30B-Thinking & 1.2 & 20 & 0.95 & 0.7 & 20 & 0.95 \\
GPT-OSS-20B (high) & 1.2 & none & 1.0 & 1.0 & 20 & 0.95 \\
\bottomrule
\end{tabular}
\vspace{0.8em}
\caption{\textbf{Baseline decoding settings used for the frozen-model comparisons in this paper.} These are the model-specific sampling configurations used for the base-model results reported in the main text.}
\label{tab:official_sampling}
\small
\setlength{\tabcolsep}{4pt}
\begin{tabular}{l ccc}
\toprule
\textbf{Model} & $T$ & top-$k$ & top-$p$ \\
\midrule
Llama-3.1-8B-Instruct       & 0.9 & none & 0.8  \\
Qwen3-4B-Instruct-2507      & 0.7 & 20 & 0.8  \\
Qwen3-4B-Thinking-2507      & 0.6 & 20 & 0.95 \\
Qwen3-30B-A3B-Instruct-2507 & 0.7 & 20 & 0.8  \\
Qwen3-30B-A3B-Thinking-2507 & 0.6 & 20 & 0.95 \\
GPT-OSS-20B (high)          & 1.0 & none & 1.0  \\
\bottomrule
\end{tabular}
\end{table}

\textbf{Prompt formatting and optimization.}
All models are queried in a single-turn chat format using their official chat templates. For instruct models, each problem is wrapped in a system message requesting a Python solution inside a markdown code block, followed by the problem statement as the user turn. For thinking models, we keep the same task presentation but do not add an explicit chain-of-thought instruction; instead, we rely on the model's native template to trigger its built-in reasoning behavior. Fine-tuning uses Megatron-LM\footnote{\url{https://github.com/NVIDIA/Megatron-LM}} on 8$\times$B200 GPUs, with expert parallelism EP${=}8$ for the MoE models. We optimize with AdamW ($\beta_1{=}0.9$, $\beta_2{=}0.95$, weight decay $0.1$) and cosine learning-rate decay from $5 \times 10^{-6}$ to $1 \times 10^{-6}$, using global batch size 32 and sequence length 65{,}536. Instruct models are trained for 2{,}500 iterations and thinking models for 300 iterations; checkpoints are saved every 250 and 50 iterations respectively.

\textbf{Evaluation and decoding settings.}
Our primary benchmark is LiveCodeBench\,v6 (LCB\,v6; 131 problems, February to May 2025), stratified by easy, medium, and hard difficulty. We also report LiveCodeBench\,v5 (374 problems, August 2024 to February 2025) as a secondary confirmation on a larger set. The primary metric throughout the paper is pass@1, and we additionally report pass@5 and per-difficulty breakdowns. All pass@$k$ estimates use 10 independent samples per problem. Frozen base-model baselines are evaluated with the model-specific baseline decoding settings in \Cref{tab:official_sampling}, while post-\SD{} models are evaluated with the student-side decoding settings in \Cref{tab:sd_settings}. Taken together, these details fully instantiate the compact setup described in \Cref{sec:experiments}.

\subsection{How \SD{} Hyperparameters Interact: Full Sweeps}
\label{app:hyperparameter_sweep}

This subsection expands the temperature-interaction results from \Cref{sec:temperature}. To keep the large ablation sweep tractable, these runs use fewer training steps than the main experiments. The main text shows representative views of the effect; here we give the fuller sweeps that make the same structure visible across both pass@1 and pass@5.

\textbf{Qwen3-30B-Instruct full sweep.}
\Cref{fig:heatmap} expands the main-text scatter plots by showing the broader search over training-time and evaluation-time decoding configurations for Qwen3-30B-Instruct.

\begin{figure}[t]
\centering
\includegraphics[width=\textwidth]{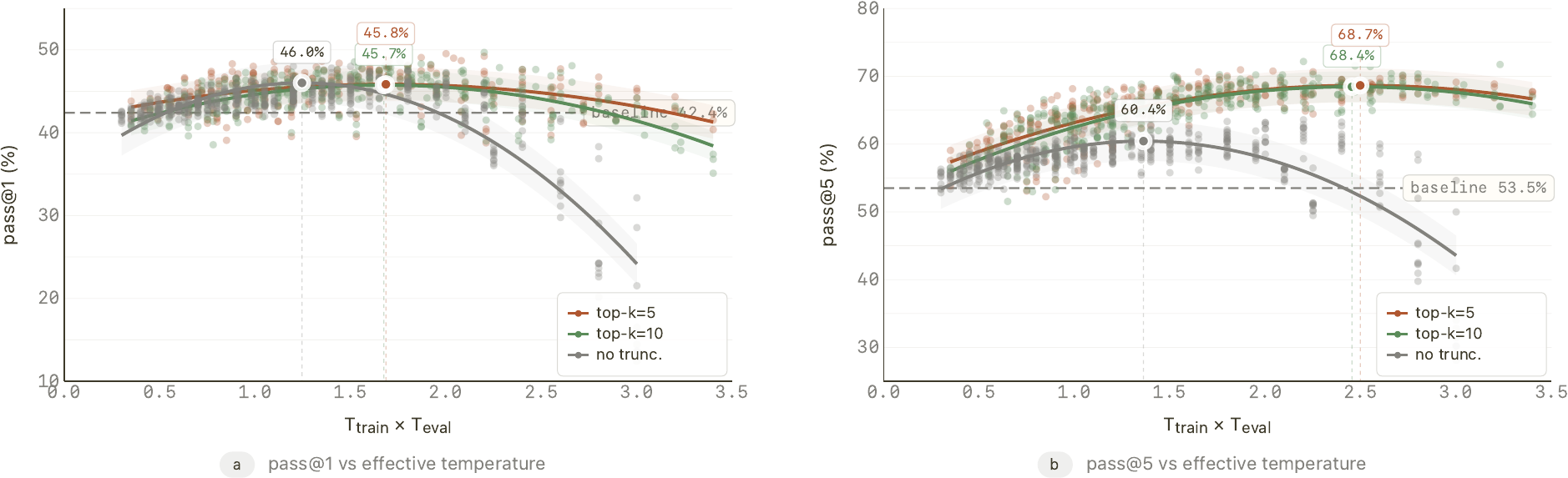}
\caption{\textbf{Across the full sweep, truncated and no-truncation settings occupy similar broad operating bands, but truncation reaches a higher pass@1 ceiling.} \SD{} hyperparameter search on LCB\,v6 for Qwen3-30B-Instruct (baseline: 42.4\% pass@1 and 53.5\% pass@5). Panels show representative configurations against effective temperature $T_{\textsf{eff}}$; curves are per-group quadratic fits, and the dashed line marks the frozen instruct baseline.}
\label{fig:heatmap}
\end{figure}

\textbf{Reading the full sweep.}
Two patterns matter. First, the successful configurations occupy a broad band rather than a single fragile optimum, which supports the claim that training-time and evaluation-time temperatures interact through a relatively stable operating region. Second, the truncated runs consistently achieve a higher pass@1 ceiling than the no-truncation runs, indicating that training-time support compression contributes something beyond temperature composition alone.

\textbf{No-truncation results and effective temperature.}
\Cref{fig:heatmap_no_trunc} isolates the no-truncation regime, where the temperature-composition picture is cleanest. In this setting, pass@1 and pass@5 are largely organized by the effective temperature $T_{\textsf{eff}} = T_{\textsf{train}} T_{\textsf{eval}}$: configurations with similar products achieve similar performance even when the two temperatures are factored differently. The broad peak near $T_{\textsf{eff}} \approx 1.2$ is consistent with the composition analysis developed in \Cref{sec:temperature,app:temp_composition}.

\textbf{Thinking-model confirmation.}
\Cref{fig:heatmap_thinking} shows that the same qualitative structure appears in Qwen3-4B-Thinking. The best-performing region again lies on a moderate diagonal band rather than at isolated values of either temperature alone. This matters because it shows that the temperature-composition pattern is not confined to instruct-style models. Taken together, the full sweeps support the two-part claim from \Cref{sec:temperature}: without training-time truncation, performance is largely organized by effective temperature, while truncation preserves that broad structure and lifts the achievable ceiling.

\begin{figure}[t]
\centering
\includegraphics[width=\textwidth]{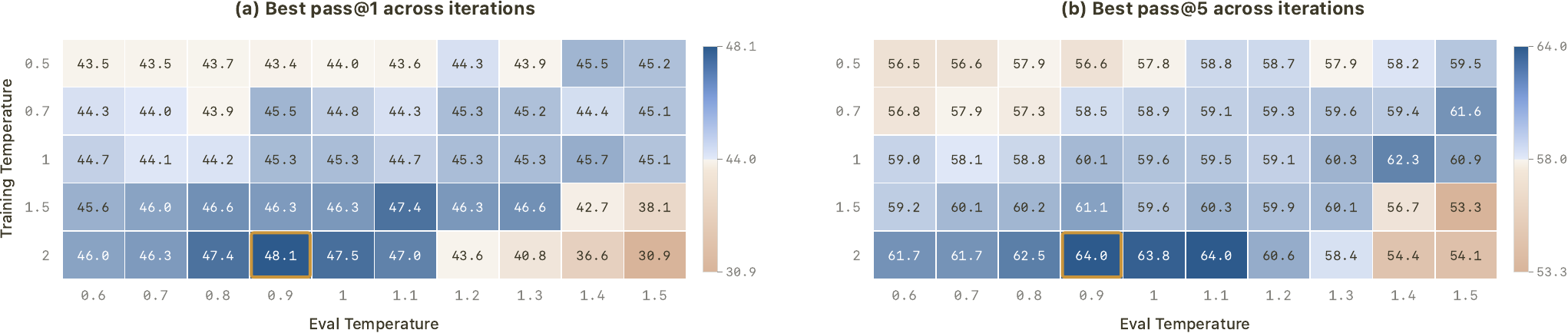}
\caption{\textbf{Without training-time truncation, performance is largely organized by the effective temperature $T_{\textsf{eff}} = T_{\textsf{train}} T_{\textsf{eval}}$.} Qwen3-30B-Instruct without top-$k$/top-$p$ truncation during \SD{} data generation. Panels show best pass@1 (left) and pass@5 (right) across the full grid of $(T_{\textsf{train}}, T_{\textsf{eval}})$ settings.}
\label{fig:heatmap_no_trunc}
\end{figure}

\begin{figure}[t]
\centering
\includegraphics[width=\textwidth]{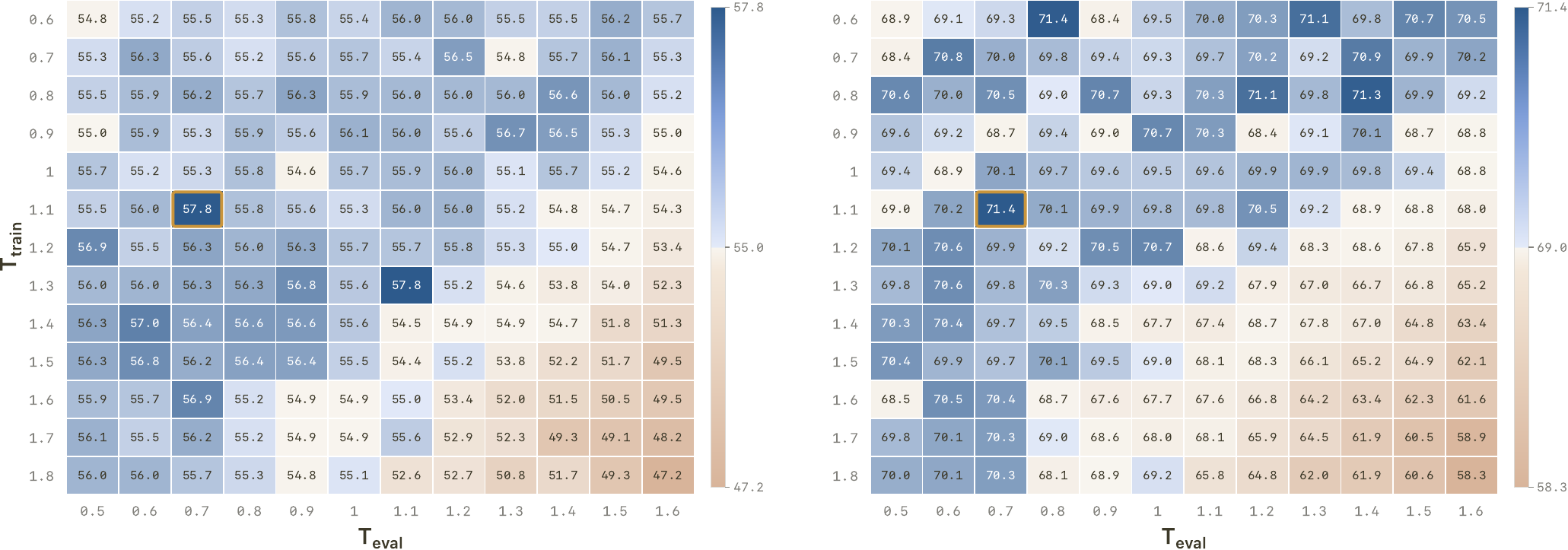}
\caption{\textbf{The same effective-temperature structure appears in Qwen3-4B-Thinking.} Best pass@1 (left) and pass@5 (right) across the full grid of $(T_{\textsf{train}}, T_{\textsf{eval}})$ settings without training-time truncation.}
\label{fig:heatmap_thinking}
\end{figure}

\subsection{Out-of-Domain Transfer}
\label{app:more_experimental_results}
\label{sec:transfer}

This subsection makes precise the main-text claim that \SD{} training on competitive-programming prompts does not substantially damage broader capabilities. We evaluate transfer on benchmarks for math reasoning, general code generation, and code understanding, and the resulting picture is scale-dependent rather than uniform.

\textbf{Benchmark scope.}
We use AIME to probe mathematical reasoning, HumanEval~\citep{chen2021evaluating} to probe general code generation in Python and Shell, and CruxEval to probe code understanding, and MMLU~\citep{hendrycks2021measuring} to probe general knowledge. These benchmarks are adjacent to competitive programming but not identical to it, which makes them a useful test of whether the student is merely over-specialized to the training domain.

\textbf{The 30B models remain broadly stable.}
The clearest pattern in \Cref{tab:transfer_results} is that the two 30B models remain broadly stable under programming-only \SD{} training. For Qwen3-30B-Instruct, all changes remain within a narrow band of roughly $\pm 2$\,pp. Qwen3-30B-Thinking shows the same behavior, with only small benchmark-level shifts and no broad collapse in capability. Both 30B models also maintain their MMLU scores to within 0.3\,pp. For the largest models in our study, programming-only \SD{} therefore appears to preserve non-competitive-programming performance reasonably well.

\textbf{Smaller models show more uneven tradeoffs.}
At smaller scale, the picture becomes more mixed. Qwen3-4B-Instruct shows clearer regressions on AIME~'24 and HumanEval Shell, even though it remains stable or slightly improved on HumanEval Python and CruxEval. Qwen3-4B-Thinking shows a different profile, with small declines on some benchmarks but substantial gains on CruxEval. Llama-3.1-8B-Instruct exhibits the sharpest tradeoff, losing ground on AIME while improving on HumanEval and CruxEval. By examining Llama's generations on AIME, we found that the model frequently fails to output a final numerical answer and instead produces a code block, leading to near-zero accuracy. The transfer story is therefore best summarized as scale-dependent: the 30B models remain broadly stable, whereas the smaller models show more uneven benchmark-specific tradeoffs.

\begin{table}[t]
\caption{\textbf{Programming-only \SD{} preserves out-of-domain performance well for the 30B models, while smaller models show more uneven benchmark-dependent tradeoffs.} Transfer results across math reasoning (AIME), general code generation (HumanEval), and code understanding (CruxEval), and general knowledge (MMLU), reported as percentages. Best within each model group is shown in \textbf{bold}.}
\label{tab:transfer_results}
\centering
\small
\setlength{\tabcolsep}{3pt}
\begin{tabular}{l cc cc cc c}
\toprule
& \multicolumn{2}{c}{\textbf{AIME}} & \multicolumn{2}{c}{\textbf{HumanEval}} & \multicolumn{2}{c}{\textbf{CruxEval}} & \textbf{MMLU} \\
\cmidrule(lr){2-3} \cmidrule(lr){4-5} \cmidrule(lr){6-7}
\textbf{Model} & '24 & '25 & Py & Sh & Input & Output & Overall \\
\midrule
Llama-3.1-8B-Instruct & \textbf{4.7} & \textbf{2.0} & 65.2 & 25.3 & 33.9 & 5.9 & \textbf{68.4} \\
\quad + \SD{} & 0.3 & 0.3 & \textbf{67.1} & \textbf{29.1} & \textbf{38.8} & \textbf{10.1} & 68.3 \\
\midrule
Qwen3-4B-Instruct & \textbf{61.3} & \textbf{44.6} & 87.8 & \textbf{47.5} & 85.5 & \textbf{77.6} & \textbf{70.6} \\
\quad + \SD{} & 55.0 & 43.3 & \textbf{88.4} & 31.7 & \textbf{86.8} & \textbf{77.6} & 70.0 \\
\midrule
Qwen3-4B-Thinking & \textbf{85.0} & 76.7 & 86.0 & \textbf{41.1} & 43.0 & 39.5 & \textbf{68.7} \\
\quad + \SD{} & 84.3 & \textbf{77.7} & \textbf{87.2} & 38.6 & \textbf{52.4} & \textbf{49.6} & \textbf{68.7} \\
\midrule
Qwen3-30B-Instruct & 77.3 & \textbf{58.5} & \textbf{95.1} & \textbf{51.3} & 90.3 & \textbf{82.9} & \textbf{80.2} \\
\quad + \SD{} & \textbf{78.9} & 58.1 & 94.5 & 50.6 & \textbf{91.6} & 81.1 & 80.1 \\
\midrule
Qwen3-30B-Thinking & \textbf{91.3} & \textbf{80.5} & 94.5 & \textbf{52.5} & \textbf{94.4} & \textbf{92.4} & \textbf{78.7} \\
\quad + \SD{} & 90.6 & 80.1 & \textbf{95.1} & 50.6 & 92.5 & 91.6 & 78.4 \\
\bottomrule
\end{tabular}
\end{table}

\FloatBarrier
\subsection{Toy Simulation: Full Specification and Additional Analyses}
\label{app:simulation_details}

This subsection provides the full specification of the toy environment introduced in \Cref{sec:mechanism:toy}. The aim is to make each part of the main-text mechanism story explicit: the FSM structure, the student induced by \SD{}, the global temperature sweep, the robustness to truncation choice, and the fork-level operational policy at the optimum.

\textbf{Why this toy is informative.}
The controlled simulation is designed so that every successful trajectory must traverse both kinds of contexts that matter for the paper's hypothesis: a fork, where several continuations remain plausible, and a sequence of locks, where only one continuation is correct but distractor mass remains in the tail. Because every transition is specified explicitly, success probability can be computed exactly under any decoding temperature and truncation setting.

\textbf{FSM specification and state archetypes.}
The toy uses a finite-state machine with vocabulary $V{=}16$. The root branches into two symmetric successful paths, each of which traverses one fork state followed by $L{=}3$ lock states before reaching PASS. At non-root states, tok0 is the unique correct continuation and all other tokens lead to FAIL. At the root, tok0 and tok1 enter the two successful paths, but tok2 is the highest-probability token and immediately fails. The three distribution archetypes are chosen to realize three distinct regimes: a fail-dominated root, a broad fork-like head, and a sharply peaked lock with a diffuse distractor tail. \Cref{fig:v19_toy_world} and \Cref{tab:toy_fsm_spec} make this structure concrete.

\begin{table}[t]
\centering
\small
\caption{\textbf{State archetypes in the toy FSM.}
The three state types instantiate a fail-dominated root, a broad fork-like head, and a sharply peaked lock with a diffuse distractor tail. In each row, the listed values are the highest-probability tokens; the remaining 12 tokens follow a geometric tail that sums to the residual mass.}
\label{tab:toy_fsm_spec}
\begin{tabular}{llp{6.5cm}}
\toprule
\textbf{State type} & \textbf{Head probabilities} & \textbf{Character} \\
\midrule
Lock ($\times 3$ per path) & $[0.750,\; 0.055,\; 0.050,\; 0.037,\;\ldots]$ & Correct at rank-1; 25\% of mass remains in a 15-token geometric tail \\
Fork ($\times 1$ per path) & $[0.148,\; 0.280,\; 0.140,\; 0.144,\;\ldots]$ & Correct at rank-2; four head tokens are near-tied \\
Root ($\times 1$) & $[0.200,\; 0.190,\; 0.329,\; 0.126,\;\ldots]$ & FAIL token (tok2) dominates; the two correct paths begin at ranks 2 and 3 \\
\bottomrule
\end{tabular}
\end{table}

\begin{figure}[t]
\centering
\includegraphics[width=\textwidth]{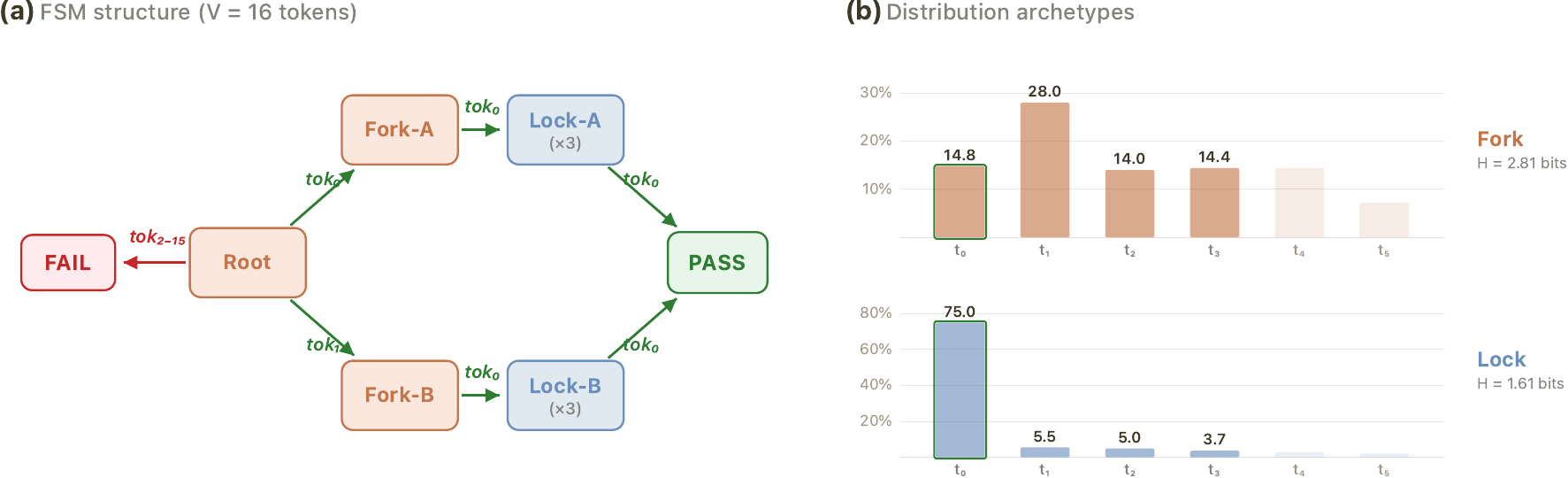}
\caption{\textbf{The toy FSM makes the lock/fork conflict explicit.} \textbf{(a)} The root branches into two symmetric successful paths, each of which traverses one fork and three locks before PASS. At non-root states, tok0 is the correct continuation and all other tokens lead to FAIL; at the root, the highest-probability token is incorrect. \textbf{(b)} The associated token-distribution archetypes instantiate a broad fork-like head and a sharp lock-like head with a distractor tail.}
\label{fig:v19_toy_world}
\end{figure}

\begin{figure}[t]
\centering
\includegraphics[width=\textwidth]{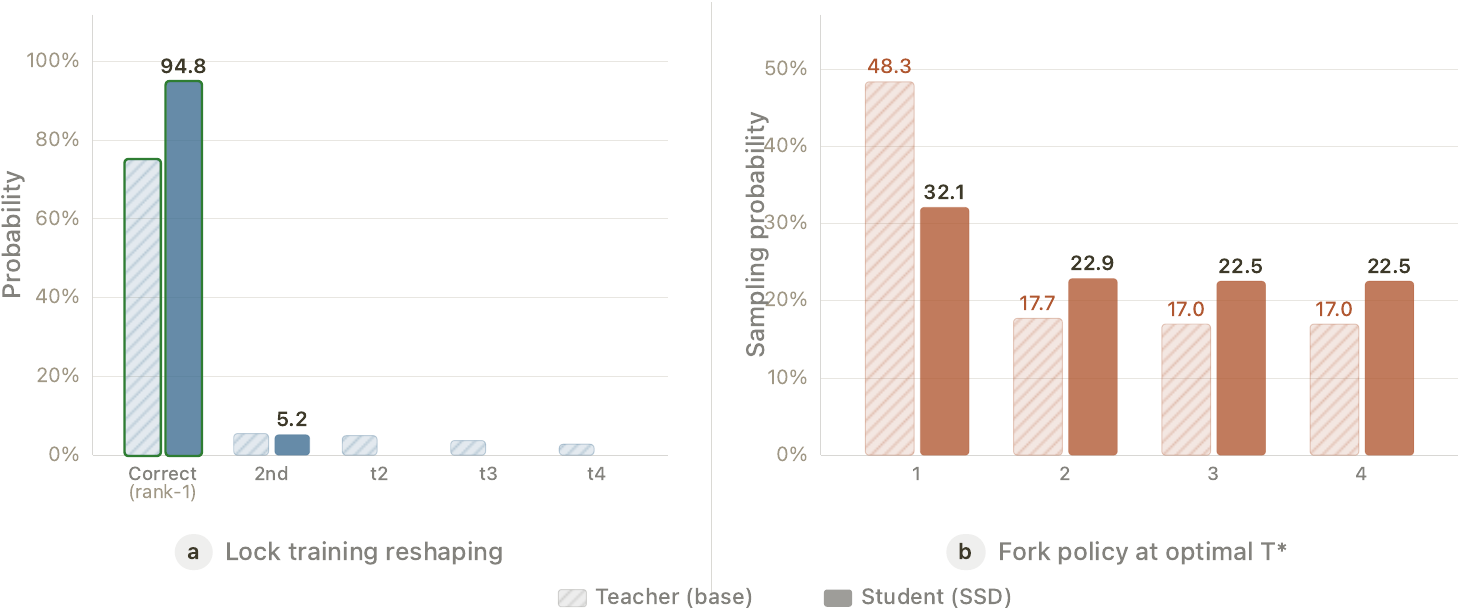}
\caption{\textbf{SSD reshapes lock distributions and flattens fork policies.}
\textbf{(a)}~Lock training reshaping: hatched bars show the teacher base distribution $p_0$, solid bars show the student $p_\theta$ after \SD{}. The correct token absorbs nearly all mass (94.8\%).
\textbf{(b)}~Fork operational policy at each model's optimal $T$ with top-$p{=}0.80$: the teacher has a descending four-token nucleus, while the student yields a flatter plateau that allocates more mass to the correct lower-ranked continuation.}
\label{fig:v19_training_spike}
\end{figure}

\textbf{The induced student is already asymmetric.}
Applying \SD{} in the toy with $T_{\textsf{train}}{=}0.9$ and top-$p{=}0.85$ produces a student whose retained support differs sharply across contexts. Locks collapse to a 2-token support: the correct token absorbs 94.8\% of mass with only a single runner-up at 5.2\%, while the remaining 14 tokens are pruned entirely (\cpanel{fig:v19_training_spike}{a}). Forks retain a broader 5-token support headed by tok1 at 34.4\%, with the correct token (tok0) at 16.9\% and three near-tied alternatives around 16\%; the other 11 tokens are removed. At the root, 4 tokens survive with the fail token still dominant at 40.6\%. This is the first key point of the toy: a single global training rule induces context-dependent reshaping automatically.

\textbf{The global optimum shifts upward after \SD{}.}
We evaluate the toy by sweeping a single global decoding temperature while fixing top-$p{=}0.80$. Because the FSM is known exactly, the resulting success probability can be computed in closed form:
\[
P = \bigl[q_{\text{root}}(\text{A}) + q_{\text{root}}(\text{B})\bigr] \cdot q_{\text{fork}}(\text{correct}) \cdot q_{\text{lock}}(\text{correct})^3,
\]
where each $q$ denotes the operational (post-truncation, post-temperature) probability of the correct continuation at the corresponding state. The teacher's best global success probability is 8.32\% at $T{=}0.639$; the student reaches 13.77\% at $T{=}2.091$, a gain of +5.4\,pp with the optimal temperature shifting roughly $3\times$ upward. At their respective optima, both models retain a four-token fork nucleus, but the teacher's is steeply descending ($[48.2,\, 17.8,\, 17.0,\, 17.0]\%$) while the student's is a near-uniform plateau ($[32.1,\, 22.9,\, 22.5,\, 22.5]\%$), allocating substantially more mass to the correct lower-ranked continuation. This is the toy analogue of the mechanism claim in the main text: after \SD{}, lock states become more resistant to evaluation-time temperature, so decoding can spend more of its budget on useful exploration at the fork. \Cref{fig:v19_storyboard} shows this shift directly across three temperature regimes.

\begin{figure}[p]
\centering
\includegraphics[width=\textwidth]{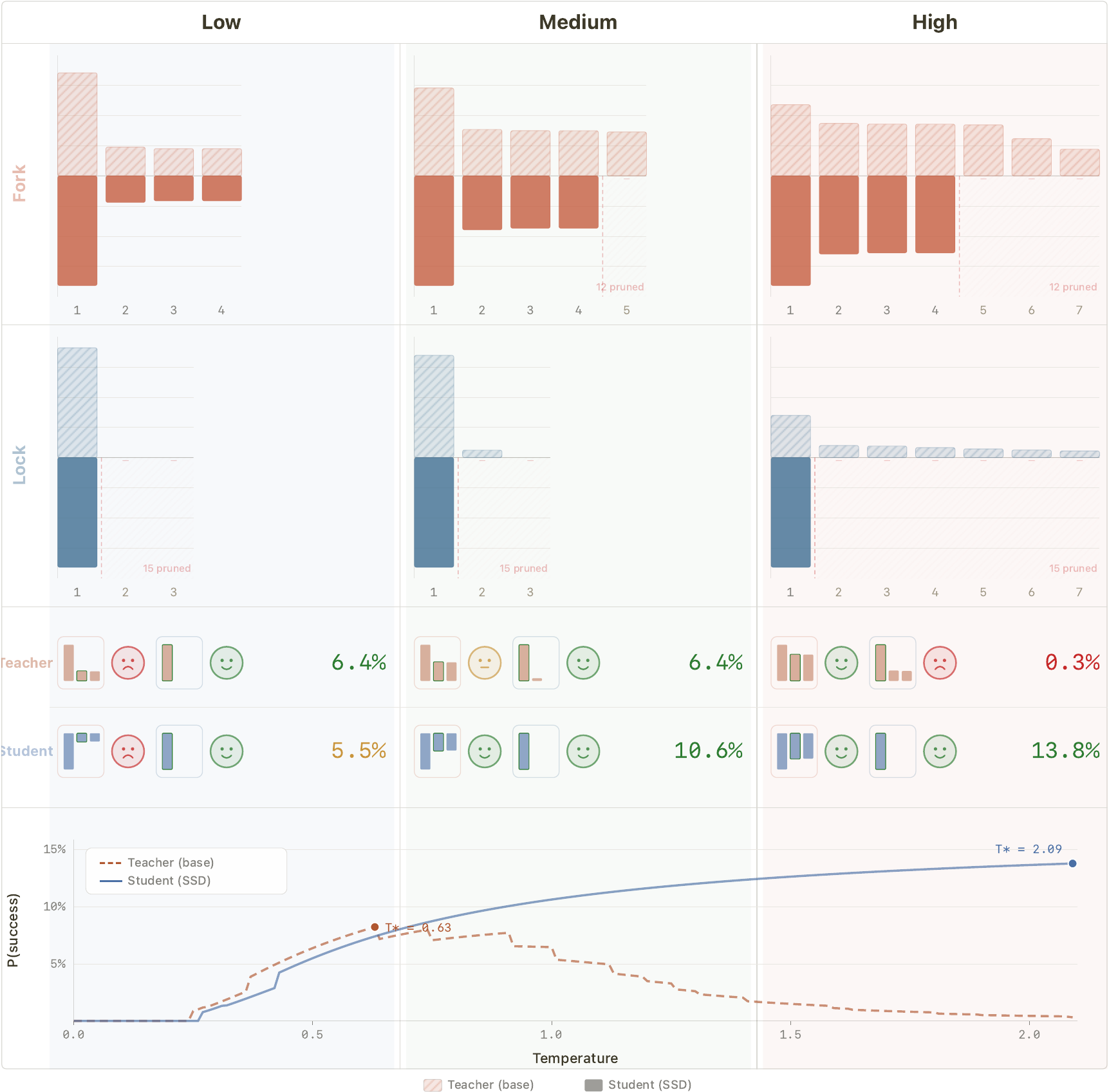}
\caption{\textbf{After \SD{}, the toy's globally optimal decoding regime shifts upward because locks become more temperature-inert while forks remain exploitable.}
The figure presents a complete view of the V19 toy simulation across three temperature regimes (columns: Low, Medium, High), organized as butterfly charts where teacher bars extend upward from the midline and student bars extend downward, with tokens ranked by descending teacher probability.
\textbf{Fork row:} At low temperature, the teacher's four-token head is peaked and the student's nucleus is narrower but similarly shaped. At medium temperature, both distributions broaden, and the student retains more mass on the correct lower-ranked continuation (rank~2). At high temperature, the student's plateau policy allocates mass nearly uniformly across the retained head, while the teacher's tail remains diffuse.
\textbf{Lock row:} The lock distribution is sharply peaked at the correct token (rank~1) under both teacher and student across all three temperatures. Training-time support compression collapses the student's lock to a 2-token support (94.8\% on correct), making it nearly temperature-inert.
\textbf{Readout:} For each temperature column, the readout shows per-state sparkline summaries and satisfaction indicators for both teacher and student, alongside the end-to-end success probability $P = [q_{\textsf{root}}(\textsf{A}) + q_{\textsf{root}}(\textsf{B})] \cdot q_{\textsf{fork}}(\textsf{correct}) \cdot q_{\textsf{lock}}(\textsf{correct})^3$. The student's advantage grows with temperature: from comparable at low~$T$ (5.5\% vs 6.4\%) to substantially better at high~$T$ (13.8\% vs 0.3\%).
\textbf{Curve:} The bottom panel plots exact success probability as a function of the global decoding temperature. The teacher (dashed) peaks at $T^*{=}0.63$ and declines sharply, while the student (solid) peaks at $T^*{=}2.09$ and remains competitive across a wide band.}
\label{fig:v19_storyboard}
\end{figure}

\textbf{The advantage is robust and visible at the fork.}
The student advantage is not an artifact of one carefully chosen truncation threshold. Repeating the grid search across top-$p \in \{0.65, 0.70, 0.75, 0.80, 0.85, 0.90\}$ leaves the student ahead throughout, with gaps ranging from +1.4\,pp (top-$p{=}0.90$) to +5.4\,pp (top-$p{=}0.80$). The same asymmetry also appears directly in the fork operational policy (\cpanel{fig:v19_training_spike}{b}): at each model's own optimum, the teacher has a descending four-token nucleus, while the student's is much closer to a plateau and assigns more mass to the correct lower-ranked continuation. Taken together, these analyses support all three qualitative pieces of the main-text story: safer locks, more usable fork-level diversity, and a higher globally optimal decoding regime after training.

\subsection{High-Temperature Case Study: Full Details and Additional Analyses}
\label{app:high_temp_details}

\Cref{sec:high_temp} presents a stress test in which \SD{} training uses $T_{\textsf{train}}{=}2.0$ with no top-$k$ or top-$p$ truncation. The purpose of the case study is to ask whether \SD{} still helps when the sampled training outputs are overwhelmingly poor as programs.

\textbf{Why this case matters.}
This setting directly tests a plausible alternative explanation for the paper's gains, namely that \SD{} works mainly because it trains on sampled programs that are already fairly good. By pushing the training distribution into a regime where that explanation should fail, the case study isolates the contribution of distributional reshaping from the contribution of superficial sample quality.

\textbf{The training corpus is deliberately poor.}
We generate one sample per prompt from Qwen3-30B-Instruct at $T_{\textsf{train}}{=}2.0$ with both top-$k$ and top-$p$ disabled, using the same prompt pool as in the main experiments. All outputs are retained without filtering. The resulting corpus is visibly poor: across the generation shards, only about ${\sim}37\%$ of outputs contain a chain-of-thought followed by an extractable code block, while about ${\sim}62\%$ contain no extractable code at all. Even seemingly coherent outputs often devolve into multilingual gibberish mid-sequence. By ordinary data-quality standards, this is far worse than the truncated setting used in the main experiments. Training otherwise uses the same infrastructure as the main Qwen3-30B-Instruct experiments, and the final training loss rises to 11.29, reflecting the much noisier targets.

\textbf{The student still improves across a broad region.}
Despite the poor training corpus, the resulting student improves materially. We evaluate every saved checkpoint from iterations 250 through 2{,}500 across ten values of $T_{\textsf{eval}} \in [0.6, 1.5]$, always using evaluation-time top-$k{=}20$ and top-$p{=}0.95$. This yields 100 checkpoint and temperature configurations in total, each evaluated with 10 repetitions. Of these, 62 exceed the 42.4\% frozen-base pass@1 baseline. The best configuration reaches 48.1\% pass@1 and 64.0\% pass@5, and the gains again concentrate on the hard subset. The key qualitative point is that the optimum is not a single isolated lucky cell: it lies inside a contiguous late-training ridge at low-to-moderate $T_{\textsf{eval}}$.

\begin{figure}[t]
\centering
\includegraphics[width=\textwidth]{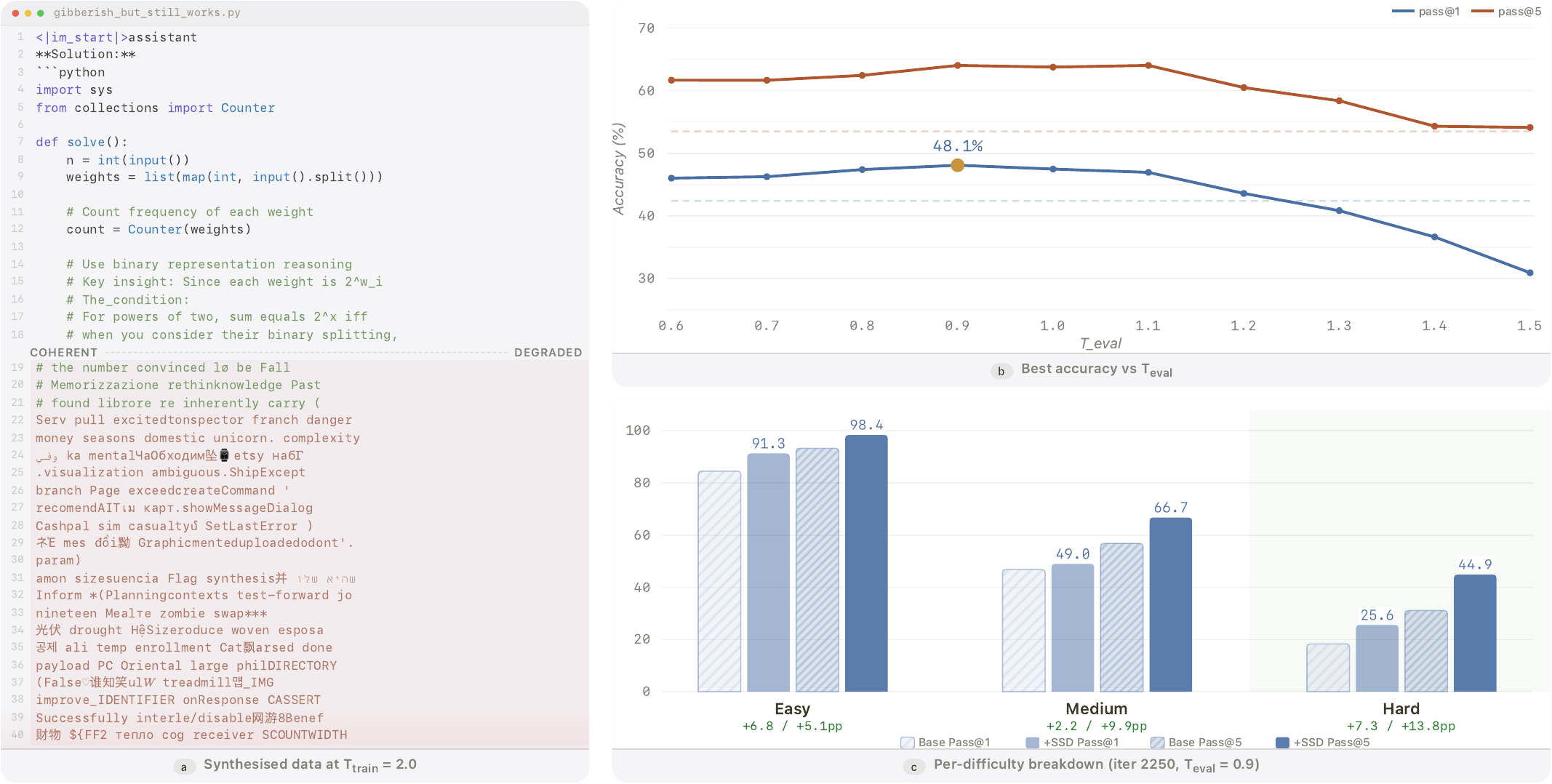}
\caption{\textbf{Expanded view of the high-temperature stress test ($T_{\textsf{train}}{=}2.0$, no truncation) on Qwen3-30B-Instruct.}
\textbf{(a)}~A representative training sample: the first 12 lines contain coherent Python, but the output degrades into multilingual gibberish by line~13; approximately 62\% of synthesized outputs contain no extractable code at all.
\textbf{(b)}~Best pass@1 (blue) and pass@5 (rust) across all checkpoints for each evaluation temperature $T_{\textsf{eval}} \in [0.6, 1.5]$; dashed lines mark the 42.4\% and 53.5\% base-model baselines. The peak at $T_{\textsf{eval}}{=}0.9$ reaches 48.1\% pass@1 and 64.0\% pass@5.
\textbf{(c)}~Per-difficulty breakdown at the best checkpoint (iteration~2250, $T_{\textsf{eval}}{=}0.9$): hatched bars show the base model, solid bars show +\SD{}. Gains concentrate on harder problems: easy +6.8/+5.1\,pp, medium +2.2/+9.9\,pp, hard +7.3/+13.8\,pp (pass@1/pass@5).}
\label{fig:high_temp_expanded}
\end{figure}

\textbf{The viable region is bounded, and that matters.}
The same grid also shows that the successful region is sharply bounded. Performance remains competitive for $T_{\textsf{eval}}$ roughly in the range $[0.6, 1.1]$, but degrades quickly once evaluation-time temperature becomes too high, falling below baseline at $T_{\textsf{eval}}{=}1.3$ and dropping further at $T_{\textsf{eval}}{=}1.5$. This pattern is consistent with the temperature-composition picture developed earlier in the paper. In this regime, training approximates a high-temperature reshaping of the teacher without training-time support compression, so evaluation succeeds only while the resulting effective temperature remains inside a viable band.

\textbf{Comparison with the standard truncated recipe.}
The stress test remains visibly weaker than the standard truncated \SD{} recipe, and that gap is itself informative. When truncation is present during training, support compression is active throughout optimization and directly suppresses distractor tails in the student. In the present case, those tails are not suppressed during training and must instead be cleaned up at evaluation time by top-$k$/top-$p$ truncation. The gains therefore remain real but smaller and more fragile. Taken together, this case study supports a narrower but important conclusion: even when the sampled programs are mostly poor, \SD{} can still help because the useful signal lies in distributional reshaping rather than in raw program correctness alone.

\bibliographystyle{plainnat}
\bibliography{references}

\end{document}